\documentclass[11pt]{article}
\usepackage{amsmath,amsthm,amssymb}
\usepackage{latexsym}
\usepackage{mathrsfs}
\usepackage[normalem]{ulem}

\usepackage{enumerate}
\usepackage[shortlabels]{enumitem}
\usepackage{verbatim}

\usepackage[english]{babel}

\usepackage{graphics,graphicx}
\usepackage{wrapfig} 
\usepackage{subcaption}
\usepackage{subfloat}
\usepackage{booktabs}
\usepackage{multirow}
\usepackage{threeparttable}
\allowdisplaybreaks[3] 
\usepackage{float}

\usepackage{algorithm}
\usepackage{algorithmic}

\usepackage{bbold}

\usepackage{appendix}

\usepackage{authblk}

\usepackage{tikz}
\usetikzlibrary{arrows.meta, positioning, decorations.pathmorphing, shapes.geometric}

\usepackage{color}
\usepackage{bm}
\usepackage{url}
\usepackage{hyperref}
\hypersetup{
  colorlinks=true,
  linkcolor=blue,
  filecolor=blue,
  anchorcolor=blue,
  urlcolor=blue,
  citecolor=purple
}

\usepackage{tikz}
\usepackage{pgfplots}
\usepackage{pifont}

\usepackage[margin=1in]{geometry}

\newtheorem{theorem}{Theorem}[section]

\newtheorem{lemma}[theorem]{Lemma}
\newtheorem{prop}[theorem]{Proposition}
\newtheorem{fact}[theorem]{Fact}
\newtheorem{definition}[theorem]{Definition}
\newtheorem{assumption}[theorem]{Assumption}
\theoremstyle{definition}
\newtheorem{remark}[theorem]{Remark}
\newtheorem*{example}{Example}

\numberwithin{equation}{section}

\newcommand{\bz}{\mathbf 0}
\newcommand{\bP}{\mathbb P}
\newcommand{\Q}{\mathbb Q}
\newcommand{\R}{\mathbb R}

\DeclareMathOperator{\prox}{prox}
\DeclareMathOperator{\proj}{proj}

\DeclareMathOperator{\dist}{dist}

\DeclareMathOperator*{\argmin}{argmin}
\DeclareMathOperator*{\argmax}{argmax}

\usepackage[dvipsnames]{xcolor}

\usepackage{amssymb,amsmath,amsfonts,bm}

\def\1{\bm{1}}

\DeclareMathAlphabet{\mathsfit}{\encodingdefault}{\sfdefault}{m}{sl}
\SetMathAlphabet{\mathsfit}{bold}{\encodingdefault}{\sfdefault}{bx}{n}

\begin{document}
\title{\bf Distributionally Robust Optimization via Iterative Algorithms in Continuous Probability Spaces}
\author[1]{Linglingzhi Zhu\thanks{\href{mailto:llzzhu@gatech.edu}{llzzhu@gatech.edu}}}
\author[1]{Yunqin Zhu\thanks{\href{mailto:yzhu812@gatech.edu}{yzhu812@gatech.edu}}}
\author[1]{Yao Xie\thanks{\href{mailto:yao.xie@isye.gatech.edu}{yao.xie@isye.gatech.edu}}}
\affil[1]{H. Milton Stewart School of Industrial and Systems Engineering\authorcr Georgia Institute of Technology}
\date{July 2026}

\maketitle

\begin{abstract}
We study distributionally robust optimization (DRO) for robust inference when the worst-case distribution is continuous, leading to significant computational challenges due to the infinite-dimensional nature of the optimization problem. Unlike traditional discrete DRO approaches, which often suffer from scalability issues, limited generalization, and costly worst-case inference, our framework exploits Brenier’s theorem to characterize the least favorable distribution as the pushforward of a transport map from a continuous reference measure. This characterization motivates our study of the minimax problem in Wasserstein space. We propose an iterative algorithmic framework with multiple variants and establish global convergence guarantees under mild assumptions, deriving complexity bounds in terms of subgradient evaluations and inexact Jordan–Kinderlehrer–Otto updates. Numerical results with neural network-based transport maps demonstrate that the proposed method enables both stable training of robust classifiers and effective worst-case inference for classification tasks.
\end{abstract}

\section{Introduction}

Distributionally robust optimization (DRO) addresses inference and decision-making problems under uncertainty of the data distribution for stochastic optimization models. Since parameterized uncertainty set constructions, e.g., moments \cite{bertsimas2000moment,delage2010distributionally} and deviations \cite{chen2007robust} ambiguity sets, may encompass an overly broad range of distributions significantly different from the reference, recent research has increasingly focused on non-parametric discrepancy-based ambiguity sets \cite{ben2013robust,namkoong2016stochastic,wang2016likelihood,mohajerin2018data,blanchet2019quantifying,gao2023distributionally}, which have a tunable radius that quantifies the uncertainty level. In this paper, we focus on this category of discrepancy-based DRO:
\begin{equation}\label{Problem}
\inf_{f \in \mathcal{F}}\sup_{\mathbb{Q} \in \mathcal{B}_{\delta}(\mathbb{P})}\mathbb{E}_{\xi \sim \mathbb{Q}}[\ell(f,\xi)].
\end{equation}
Here $\bP$ is a $d$-dimensional reference distribution with a finite second moment and $\mathcal{B}_{\delta}(\mathbb{P}):=\{\Q\in\mathcal{P}_{2}: \mathcal{D}(\Q,\bP )\leq\delta\}$ for some $\delta>0$ is the uncertainty set defined by the discrepancy function $\mathcal{D}$. Meanwhile, $\mathcal{F}$ is a measurable decision function class and $\ell:\mathcal{F}\times\mathbb{R}^{d}\rightarrow\R$ is the loss function.

Although the theory and applications of DRO models have been extensively studied (see, e.g., \cite{shapiro2017distributionally,rahimian2022frameworks,kuhn2025distributionally}) and efficient solvers have been recently developed \cite{liu2025dro}, comparatively less attention has been given to the explicit generation of the worst-case distribution (also called the least favorable distribution (LFD)). One of the key differences between DRO and traditional optimization is that it seeks models that perform best under the LFD in the ambiguity set (i.e., the maximization problem in \eqref{Problem}), which makes it an infinite-dimensional optimization problem. To construct worst-case samples from \eqref{Problem}, existing methods often rely on dual formulations or perturbations of nominal support points. In particular, by leveraging semi-infinite duality \cite{shapiro2001duality}, many infinite-dimensional minimax problems can be reformulated as finite tractable programs when the reference distribution is discrete \cite{namkoong2016stochastic,mohajerin2018data}. Such strong duality results have also been extended to general optimal transport costs \cite{blanchet2019quantifying,gao2023distributionally}, allowing the LFD to be obtained by solving the dual problem. On the other hand, \cite{sinha2018certifying,blanchet2022optimal} propose to solve a Lagrangian regularization formulation of the optimal transport based distributional problem from point-to-point perturbation by stochastic gradient methods in the discrete case. 
    
However, the above-mentioned approaches that focus on discrete worst-case distributions face several challenges when directly applied to generate the LFD. First, there is a significant computational challenge. The dual formulation of DRO is not scalable to large datasets; for example, constructing the LFD under the Wasserstein metric requires solving a linear program with $\mathcal{O}(N^{2})$ decision variables, where $N$ is the total number of training data points, and the complexity of solving a linear program is typically quadratic in the number of decision variables. Such computational complexity for problems with thousands of training data points can be prohibitive, which means that current DRO formulations can usually only be used to find discrete LFDs in small-sample settings. Second, discrete LFDs limit the generalization capability of the resulting algorithm. Since the LFD is discrete and supported solely on the training dataset, the corresponding optimal detector is confined to training points. This also creates a challenge when generating samples from the LFD, as the existing point-wise perturbation method \cite{sinha2018certifying} requires retraining for each unseen data point to obtain the corresponding LFD sample.
 
Given the concerns outlined above, we are motivated to explore the DRO problem \eqref{Problem} as a minimax optimization in the continuous probability space for LFD generation. This is closely related to optimization problems on the probability space equipped with Wasserstein geometry, e.g., \cite{salim2020wasserstein,wang2020information,kent2021modified,bonet2024mirror}. In particular, from a computational perspective, \cite{kent2021modified} studied the Frank--Wolfe algorithm in Wasserstein space, whose subproblem is insightfully designed as a linear program over a Wasserstein ball, a problem well studied in the DRO literature \cite{mohajerin2018data,gao2023distributionally,yue2022linear}. However, these methods typically rely on particle-based approximations, which require costly gradient computations to perturb each new data point in the DRO setting. When the inference system, such as a classifier, involves a complex neural network architecture, these methods incur additional overhead due to the need for auto-differentiation with respect to the inputs. Moreover, generating new samples from the LFD still requires repeating the expensive training process, limiting their practicality in real-time or large-scale applications.

In this paper, we design an algorithmic framework for distributionally robust inference that enables efficient worst-case generation. Assuming the solution of the inner maximization problem is attainable, we focus on the following regularized formulation of \eqref{Problem}:
    \begin{equation}
        \label{Problem-r}\tag{P}
        \min_{\phi \in \Phi}\max_{\Q \in \mathcal{P}_2}\mathbb{E}_{\xi \sim
        \mathbb{Q}}[\ell(f_{\phi}, \xi)]-\lambda\cdot \mathcal{D}(\Q,\bP).
    \end{equation}
Here, we incorporate a parameterized decision function $f_\phi$ with $\phi \in \Phi \subseteq \mathbb{R}^m$ and a predetermined Lagrangian multiplier $\lambda \ge 0$. When the discrepancy is the Wasserstein distance, the inner maximization reduces to a Wasserstein proximal problem, as commonly studied in score-based models \cite{zhang2024wasserstein}. Different from traditional discrete DRO approaches and sampling-based methods, we leverage the existence of an optimal transport (OT) map guaranteed by Brenier's theorem in the continuous setting. This allows us to derive an equivalent optimization problem over the transport map:
    \begin{equation*}
        \min_{\phi \in \Phi}\max_{T \in L^2(\bP)}\mathbb{E}_{\xi \sim \bP}[\ell(f_{\phi}, T(\xi
        ))]-\lambda\cdot \mathcal{D}(T_{\#}\bP,\bP).
    \end{equation*}
Once the minimax-optimal transport map is learned for problem \eqref{Problem-r}, worst-case samples can be efficiently generated by passing inputs through the learned map. This reformulation enables a practical implementation via flow-based neural network parameterizations \cite{xu2024flow}, which rely on the change-of-variables formula and explicitly model data distributions through transport maps.

With this transport map reformulation in hand, we focus on the convergence guarantees of iterative algorithms for solving \eqref{Problem-r} in infinite-dimensional spaces. We propose an iterative minimax algorithmic framework (Algorithm~\ref{alg-subgdmax}) operating in Wasserstein space, which for step $k\ge0$ performs:
    \begin{equation*}
\left\{
\begin{aligned}
&\Q_{k}:=\underset{\mathbb{Q} \in \mathcal{P}_2}
            {\argmax}\left\{\mathbb{E}_{\xi \sim \mathbb{Q}}[\ell(f_{\phi_k}, \xi
            )]-\lambda\cdot \mathcal{D}(\Q,\bP)\right\}; \\
            &\phi_{k+1}:=\proj
            _{\Phi}(\phi_{k}- \eta_k\cdot\partial_{\phi}\mathbb{E}_{\xi \sim \Q_k}[\ell(f_{\phi_{k}}
            ,\xi)]).
\end{aligned}
\right.
\end{equation*}
We prove that this framework converges globally to a stationary point of \eqref{Problem-r} under mild assumptions that only require weak convexity of the loss function, covering a wide range of modern machine learning models.
 
To further enhance performance, we develop two algorithmic variants. The first is a proximal regularized algorithm inspired by the Jordan--Kinderlehrer--Otto (JKO) scheme (Algorithm~\ref{alg-dualppm}), which iteratively computes the LFD in the inner maximization problem. 
The proximal regularization stabilizes the neural transport-map update by discouraging large distributional jumps between successive iterates, which improves the exploration of multimodal LFDs. When the discrepancy function $\mathcal{D}$ satisfies strong convexity along generalized geodesics (a.g.g.) with a sufficiently large regularization parameter $\lambda$, we establish exponential convergence of the inner loop. The second variant is an alternating algorithm (Algorithm~\ref{alg:single}) that, under an additional smoothness assumption, updates the decision function and the distribution in turn, performing only a single JKO step per outer iteration and thereby significantly reducing the computational burden of the inner-loop updates. The proposed algorithms allow us to quantify the total number of (sub)gradient evaluations and inexact JKO iterations needed to achieve approximate stationary points. 

Our analysis builds on the regularization structure of the DRO model \eqref{Problem-r} but is also of independent interest for optimization over distributions, as it provides a novel convergence analysis of proximal-type methods in Wasserstein space using geometric properties along fixed generalized geodesics. This requirement is much weaker than strong convexity along all generalized geodesics assumed in the existing literature and constitutes the only geometric structure exploited in our framework, exemplified by Wasserstein-2 regularization. Another recent work \cite{cheng2025worst} studies a related problem from a different perspective, using a transport-map formulation based on the Monge formulation of optimal transport and Brenier's theorem. In contrast, our work adopts a regularized Wasserstein formulation in which the optimization variable is a continuous distribution in the $\mathcal W_2$ space. As a result, the two approaches lead to distinct analytical frameworks and rely on different technical assumptions.

Numerical experiments illustrate the benefit of neural transport maps for worst-case generation. 
On a two-dimensional Gaussian-mixture task, both the elimination-based and alternating variants are implemented as neural transport maps and benchmarked against particle-based solvers that optimize moved samples directly. 
The JKO-regularized neural update produces particle-like exploration of multimodal least favorable distributions while enabling fast inference through a learned map. 
On CIFAR-10 feature representations, learned neural maps further provide effective initializations for particle-based optimization, reducing the cost of high-dimensional worst-case inference.

\subsection{Organization}
The rest of the paper is organized as follows.  In Section \ref{sec:alg}, we propose a minimax optimization framework to solve the DRO problem \eqref{Problem-r} and establish its global convergence rate. Section~\ref{sec:LFD} presents the JKO scheme for computing the LFD as a subproblem solver within the proposed framework, and Section~\ref{sec:smooth} establishes improved complexity results for its alternating variant under smoothness assumptions. Implementation details and numerical results on both synthetic and real datasets are presented in Sections~\ref{sec:practical-w2} and \ref{sec:numerical}. Finally, we end with some closing remarks in Section \ref{sec:conclusion}. Key definitions, preliminary results, and detailed proofs are given in the Appendices.

\subsection{Notation}
Throughout this paper, we use the standard notation. Let $\mathcal{X}\subseteq\mathbb{R}^{d}$ and denote the space $\mathcal{P}(\mathcal{X})$ of Borel probability measures over $\mathcal{X}$. Let $\mathcal{P}_{2}(\mathcal{X}):=\{\bP\in\mathcal{P}(\mathcal{X}): \int_{\mathcal{X}}\|x\|^{2}\operatorname{d}\bP(x)<\infty\}$ be the space of all distributions on domain $\mathcal{X}$ that have a finite second moment. Define $\mathcal{P}_{2}^{r}(\mathcal{X})$ as all distributions in $\mathcal{P}_{2}(\mathcal{X})$ that also have densities (absolutely continuous with respect to the Lebesgue measure). We use the notation $\mathcal{P}_{2}$ and $\mathcal{P}_{2}^{r}$ when $\mathcal{X}=\R^d$. Given a (measurable) map $T: \R^d\rightarrow\R^d$ and $\mathbb{P}$ a distribution on $\R^d$, its $L^{2}$ norm is denoted as $\|T\|_{\mathbb{P}}:=(\int_{\R^d}\|T(x)\|^{2}\operatorname{d\mathbb{P}}(x))^{1/2}$. For $\mathbb{P}\in \mathcal{P}_{2}$, we denote by $L^{2}(\mathbb{P})$ the $L^{2}$ space of vector fields, that is, $L^{2}(\mathbb{P}):=\{T:\|T\|_{\mathbb{P}}^2<\infty\}$. For $T_{1}, T_{2}\in L^{2}(\mathbb{P})$, define $\langle T_{1}, T_{2}\rangle_{\mathbb{P}}:=\int_{\R^d}T_{1}(x)^{\top} T_{2}(x) \operatorname{d \mathbb{P}}(x)$. For $T: \R^d\rightarrow\R^d$, the pushforward of a distribution $\mathbb{P}$ is denoted as $T_{\#}\mathbb{P}$, such that $T_{\#}\mathbb{P}(A)=\mathbb{P}(T^{-1}(A))$ for any measurable set $A$. If the density $p$ exists, the above mentioned notation related to $\mathbb{P}$ can also be replaced by $p$. We use $\iota_{\mathcal{X}}: \mathbb{R}^{d}\rightarrow \{0, +\infty\}$ to denote the indicator function associated with $\mathcal{X}$.  
For an extended-valued function $h:\mathbb R^d\to\mathbb R\cup\{+\infty\}$ and $\epsilon\ge0$, 
the set of all $\epsilon$-optimal solutions of $h$ is defined by $\epsilon\operatorname{-}\argmin h:=
\{x\in\mathbb R^d: h(x)\leq \inf_{y\in\mathbb R^d}h(y)+\epsilon\}$, which is reduced to the optimal solution set $\argmin h$ when $\epsilon=0$. We denote by $\partial h$ the Fr\'echet subdifferential. 
We refer to other key concepts and their properties for our analysis in Appendix \ref{sec:prelim}.

\section{Minimax Algorithmic Framework}\label{sec:alg}

In this section, we present the main algorithmic framework for worst-case generation by solving the DRO problem \eqref{Problem-r} in probability space and analyze its convergence. The detailed iterative procedure is given in Algorithm~\ref{alg-subgdmax}.

\begin{algorithm}[h]
        \caption{Minimax Algorithmic Framework for Worst-Case Generation}
        \label{alg-subgdmax}
        \begin{algorithmic}
            \STATE {\bfseries Input:} Initialization $\phi_{0}$, step size
            $\eta_k> 0$, regularization parameter $\lambda>0$, data $\xi\sim \bP$
            \FOR{$k=0$ {\bfseries to} $K-1$}
            \STATE $\Q_{k}:=\epsilon\operatorname{-}\underset{\mathbb{Q} \in \mathcal{P}_2}
            {\argmax}\left\{\mathbb{E}_{\xi \sim \mathbb{Q}}[\ell(f_{\phi_k}, \xi
            )]-\lambda\cdot \mathcal{D}(\Q,\bP)\right\}$ \STATE $\phi_{k+1}:=\proj
            _{\Phi}(\phi_{k}- \eta_k\cdot\zeta(\phi_{k},\Q_{k}))$ with $\zeta(\phi_{k}
            ,\Q_{k})\in \partial_{\phi}\mathbb{E}_{\xi \sim \Q_k}[\ell(f_{(\cdot)}
            ,\xi)](\phi_{k})$
            \ENDFOR
        \end{algorithmic}
    \end{algorithm}

From a computational perspective, the key challenge for practical implementation is solving the inexact maximization step in updating the probability measure $\Q$. In fact, with fixed decision function $f_{\phi}$, we can consider the following equivalent transport map maximization problem for non-atomic $\bP$ (from the similar argument in \cite[Lemma A.1]{xu2023normalizing}) that
\begin{equation}\label{eq:max-subproblem}
\max_{T \in L^2(\bP)}\mathbb{E}_{\xi \sim \bP}[\ell(f_{\phi}, T(\xi))]-\lambda\cdot \mathcal{D}(T_{\#}\bP,\bP).
\end{equation}
Moreover, if $\mathcal{D}=\frac{1}{2}\mathcal{W}_{2}^{2}$ and $\mathbb{P}\in\mathcal{P}_2^r$, then we can show that the problem \eqref{eq:max-subproblem} is further equivalent to the $\mathcal{W}_{2}$-proximal problem by Brenier's theorem (cf. \cite[Proposition 1]{xu2024flow}) that
\begin{equation}\label{eq:ot-prox}
\max_{T \in L^2(\bP)}\mathbb{E}_{\xi \sim \bP}\left[\ell(f_{\phi}, T(\xi))-\frac{\lambda}{2}\|T (\xi)-\xi \|^{2}\right].
\end{equation}
Hence, we can solve the inexact maximization step in Algorithm \ref{alg-subgdmax} using the above mentioned equivalent form by parameterizing the transport map.

\begin{remark}[Sampling-free inference via transport maps]
For optimization over the probability space, while standard Langevin dynamics can be efficient during training due to its connection with the Wasserstein gradient flow, our transport map-based parameterization offers distinct advantages. Particle-based methods, such as Wasserstein gradient descent and mean-field approaches, require costly gradient computations for perturbing each new data point in the DRO setting. In scenarios where the decision model (e.g., a classifier) involves a complex neural network architecture, these methods incur additional overhead from auto-differentiation with respect to inputs. In contrast, the approach \eqref{eq:max-subproblem} enables fast inference: once the optimal transport map is learned, new samples from the worst-case distribution can be generated deterministically by applying the map to samples from the source distribution, requiring only a single forward pass. Our approach bypasses per-sample gradient computations by implicitly encoding gradient information within the transport maps, thus eliminating such costs during inference. However, 
the training process is not sampling-free, as it relies on empirical data to estimate expectations and optimize the transport map via gradient-based methods, requiring sampling from the reference distribution and stochastic updates.
\end{remark}

A fundamental question regarding Algorithm~\ref{alg-subgdmax} is its global convergence, the characterization of its limit points, and its oracle complexity. These results are established in the following subsection.

\subsection{Convergence Guarantees}
We begin by introducing the key notation for the distributionally robust inference model:
\begin{itemize}
    \item {\bf Minimax loss function:} $\mathcal{H}(\phi,\Q):=\mathbb{E}_{\xi \sim \mathbb{Q}}[\ell(f_{\phi}, \xi)]-\lambda\cdot \mathcal{D}(\Q,\bP)$;
    \item {\bf Worst-case value function:} $\mathcal{V}(\phi):=\max_{\Q \in \mathcal{P}_2}\mathcal{H}(\phi,\Q)$. 
\end{itemize}
The following mild assumption is required to establish the convergence guarantees of Algorithm~\ref{alg-subgdmax}.

\begin{assumption}\label{f_assump}
The function $\ell(f_{(\cdot)}, \xi)$ is $\rho$-weakly convex and $L$-Lipschitz on the convex set $\Phi$ for any $\xi\in\R^d$.
\end{assumption}

In the remainder of this paper, we focus on weakly convex functions, which are not necessarily convex. This class encompasses a wide variety of models with Lipschitz continuous gradients, and it even includes certain non-smooth functions, both of which are commonly encountered in modern machine learning. Specifically, our analysis covers traditional convex models (e.g., classification with logistic loss, or linear regression with ridge and sparse $\ell_{1}$-regularizations). Additionally, it extends to more advanced non-convex scenarios involving nonlinear regression, support vector machines (SVM) equipped with nonlinear kernels, non-convex regularizers such as smoothly clipped absolute deviation (SCAD) and minimax concave penalty (MCP), and neural networks parameterized by smooth activation functions. We provide further illustration through two concrete examples.

\begin{example}[Robust hypothesis testing]
Given data $\xi \in \R^d$, a hypothesis test is performed between the null hypothesis $H_{0}: \xi \sim \mathbb{Q}_{0}$, where $\mathbb{Q}_{0}\in \mathcal{B}_{\delta}(\mathbb{P}_{0})$, and the alternative hypothesis $H_{1}: \xi \sim \mathbb{Q}_{1}$, where $\mathbb{Q}_{1}\in \mathcal{B}_{\delta}(\mathbb{P}_{1})$. To conduct this test, we construct a measurable scalar-valued detector function $f_\phi: \R^d\rightarrow \mathbb{R}$ with $\phi \in \Phi$. Specifically, for a given observation $\xi \in \R^d$, the detector $f_\phi$ accepts $H_{0}$ and rejects $H_{1}$ if $f_\phi(\xi) < 0$, and rejects $H_{0}$ while accepting $H_{1}$ otherwise. In this framework, the objective function $\mathcal{H}(\phi, (\mathbb{Q}_{0}, \mathbb{Q}_{1}))$ is defined to provide a bound on the sum of type I and type II errors. Specifically, consider the function $\ell = (\ell_{0}, \ell_{1})$, where $\ell_{i}: \mathcal{F}\times \R^d\rightarrow \mathbb{R}$ for $i = 0, 1$ can be defined as
\[
\ell_{i}(f_\phi, \xi) = \exp(f_\phi(\xi)), \;
\log(1 + \exp(f_\phi(\xi))), \ \text{or}\ (f_\phi(\xi) + 1)_{+}^{2}.
\]
The robust hypothesis testing is thus defined as
\[
\min_{\phi\in\Phi}\max_{\Q_i \in \mathcal{B}_{\delta}(\mathbb{P}_i),i=0,1}\mathbb{E}_{\xi \sim \Q_0}[\ell_{0}(f_\phi,\xi)]+\mathbb{E}_{\xi\sim \Q_1}[\ell_{1}(-f_\phi,\xi)].
\]
As a result, it can be formulated as the following $\mathcal{W}_{2}$-regularized version considered in this paper:
\begin{align*}
\min_{\phi\in\Phi}\max_{\Q_i \in \mathcal{P}_2, i=0,1} & \mathcal{H}(\phi,(\Q_{0}, \Q_{1}))=\ \mathbb{E}_{\xi \sim \Q_0}[\ell_{0}(f_\phi,\xi)]+\mathbb{E}_{\xi \sim \Q_1}[\ell_{1}(-f_\phi,\xi)]     -\frac{\lambda}{2}\cdot(\mathcal{W}_{2}^{2}(\Q_{0},\mathbb{P}_{0})+\mathcal{W}_{2}^{2}(\Q_{1},\mathbb{P}_{1})).
\end{align*}
By definition, if the parameterized function $f_{\phi}(\xi)$ is bounded, continuously differentiable, and possesses a Lipschitz continuous gradient with respect to $\phi$, the composition satisfies the conditions stated in Assumption~\ref{f_assump}.
\end{example}

\begin{example}[Adversarial learning]
Let $\xi = (x, y)$ with $\xi \sim \mathbb{P}$ denote a data-label pair, where $x \in \mathbb{R}^{d}$ represents a feature vector and $y \in \{-1, +1\}$ is its associated label. The distributionally robust classification problem is given by:
\begin{align*}
    \min_{\phi\in\Phi}\max_{\mathbb{Q} \in \mathcal{P}_2}\mathcal{H}(\phi,\Q) & =\mathbb{E}_{\xi \sim \mathbb{Q}}\left[ \log \left(1 + \exp(-y f_\phi(x)) \right) \right] - \frac{\lambda}{2} \cdot \mathcal{W}_{2}^{2}(\mathbb{Q}, \mathbb{P}).
\end{align*}
Here, $\ell(f_\phi, \xi) = \log(1 + \exp(-y f_\phi(x)) )$. We can similarly verify that Assumption~\ref{f_assump} holds provided that the parameterized function $f_{\phi}(x)$ is bounded, continuously differentiable, and has a Lipschitz continuous gradient with respect to $\phi$.
\end{example}

Before presenting the main convergence results, we first introduce a suitable stationarity measure to assess the algorithm's performance. This requires establishing the weakly convex property of the worst-case value function $\mathcal{V}$.

\begin{lemma}[$\rho$-weak convexity of $\mathcal{V}$]\label{lem:V-wcvx}
Suppose that Assumption \ref{f_assump} holds, then the worst-case value function $\mathcal{V}:\Phi\rightarrow \R$ is $\rho$-weakly convex.
\end{lemma}

Thanks to Lemma \ref{lem:V-wcvx}, we can adopt a stationary notion based on the Moreau envelope,
    which is commonly used in optimization for weakly convex functions. To be more
    precise, we define the Moreau envelope of $\mathcal{V}$ with $r>\rho$ for
    any $\phi\in\Phi$, denoted by $\mathcal{V}_{1/r}(\cdot)$, that
    \[
        \mathcal{V}_{1/r}(\phi):=\min_{\phi^{\prime}\in\Phi}\left\{\mathcal{V}(\phi
        ^{\prime})+\frac{r}{2}\|\phi-\phi^{\prime}\|^{2}\right\}.
    \]
    The Moreau envelope has the following important property by directly adopting
    the proof in \cite[Lemma 2.2]{davis2019stochastic}.

    \begin{lemma}[Stationarity measure]
        \label{lem:stationary} Suppose that the function $\mathcal{V}$ is $\rho$-weakly
        convex on $\Phi$. Let $r>\rho$ and denote the proximal operator
        \[
        \prox_{\mathcal{V}/r}(\phi):=\operatorname{argmin}_{\phi^{\prime}\in\Phi}
        \left\{\mathcal{V}(\phi^{\prime})+\frac{r}{2}\|\phi-\phi^{\prime}\|^{2}\right\}.
        \]
        Then
        $\mathcal{V}_{1/r}$ is $C^{1}$-smooth and $\|\nabla \mathcal{V}_{1/r}(\phi
        )\| \leq \varepsilon$ implies:
        \[
            \|\prox_{\mathcal{V}/r}(\phi)-\phi\|\leq\frac{\varepsilon}{r}\ \text{
            and }\ \min_{\zeta \in \partial (\mathcal{V}+\iota_{\Phi})(\prox_{\mathcal{V}/r}(\phi))}
            \|\zeta\| \leq \varepsilon.
        \]
    \end{lemma}

    The above lemma indicates that a small $\|\nabla \mathcal{V}_{1/r}(\phi)\|$ implies
    $\phi$ is close to a point $\prox_{\mathcal{V}/r}(\phi)$, which is
    approximately a stationary point of the original function $\mathcal{V}$. In this
    context, the approximate stationarity of $\mathcal{V}$ can be evaluated
    using the norm of the gradient of its Moreau envelope. We are now ready to present the convergence
    guarantee for Algorithm \ref{alg-subgdmax}.

    \begin{theorem}[Convergence theorem]
        \label{thm:subgconvergence}
        Suppose Assumption \ref{f_assump} holds. Then for the sequence generated by Algorithm \ref{alg-subgdmax}, there exists a $k'\in\{0,1,\ldots,K-1\}$ such
        that
        \[
            \begin{aligned}
                \|\nabla \mathcal{V}_{1/2\rho}(\phi_{k'})\|^{2}  \leq \frac{2(\mathcal{V}_{1/2\rho}(\phi_0)-\min_{\phi\in\Phi} \mathcal{V}(\phi))+2 \rho L^{2}\sum_{k=0}^{K-1}\eta_k^2}{ \sum_{k=0}^{K-1}\eta_k}+4\rho\epsilon. 
            \end{aligned}
        \]
Moreover, if the step size $\eta_k=1 / \sqrt{K}$, then there exists a $k'\in\{0,1,\ldots,K-1\}$ such
        that
        \[
            \begin{aligned}
            \|\nabla \mathcal{V}_{1/2\rho}(\phi_{k'})\|\leq \mathcal{O}(K^{-1/4}
            ) + \mathcal{O}(\sqrt{\epsilon}).
            \end{aligned}
        \]
    \end{theorem}

    \begin{remark}
        By combining Lemma \ref{lem:stationary} with the results of Theorem \ref{thm:subgconvergence},
        we demonstrate that after $K$ iterations, there exists a point $\phi_{k'}$
        (for some $k' \leq K-1$) that is near an approximate stationary point $\prox
        _{\mathcal{V}/2\rho}(\phi_{k'})$. Specifically, the point $\phi_{k'}$ satisfies
        \[
            \| \phi_{k'}-\prox_{\mathcal{V}/2\rho}(\phi_{k'})\| \leq \mathcal{O}(K^{-1/4}
            ) + \mathcal{O}(\sqrt{\epsilon}),
        \]
        where the approximate stationary point $\prox_{\mathcal{V}/2\rho}(\phi_{k'}
        )$ satisfies 
        \[
        \dist(0, \partial (\mathcal{V}+\iota_{\Phi})(\prox_{\mathcal{V}/2\rho}(\phi
        _{k'}))) \leq \mathcal{O}(K^{-1/4}) + \mathcal{O}(\sqrt{\epsilon}).
        \]
    \end{remark}

Theorem~\ref{thm:subgconvergence} characterizes the number of iterations required by Algorithm~\ref{alg-subgdmax} to reach an approximate stationary point with a desired accuracy. However, compared to optimization algorithms in Euclidean space, the primary
    distinction of our proposed algorithm framework lies in the
    approach to solving the inexact maximization step on the probability space, or
    equivalently, the transport map maximization problem \eqref{eq:max-subproblem}.
    This involves finding the LFD at each iterative step of the algorithm.
Thus, the dominant computational cost arises from solving the $\mathbb{Q}_k$-subproblem, which is generally more expensive than the $\phi_k$-update. Moreover, solving this problem directly by parameterization may lead to instability and mode collapse, phenomena often observed in probability density optimization.

\section{JKO Scheme for Least Favorable Distribution}\label{sec:LFD}

To address these challenges, we draw inspiration from the JKO scheme \cite{jordan1998variational} and derive the oracle complexity of our algorithm in terms of the number of finer-grained JKO steps required in the iterative updates over probability space. The JKO scheme provides a stable approximation of the backward discrete-time gradient flow along generalized geodesics in Wasserstein space. In our framework, we adopt this scheme to approximately solve the inexact inner maximization step in Algorithm~\ref{alg-subgdmax}.

    \begin{algorithm}
        [H]
        \caption{JKO Scheme for Least Favorable Distribution}
        \label{alg-dualppm}
        \begin{algorithmic}
            \STATE {\bfseries Input:} Initialization $\Q_0$, step size $\gamma> 0$, 
            data distribution $\bP$
            \FOR{$i=0$ {\bfseries to} $I-1$}
            \STATE $\Q_{i+1}:=
            \argmax_{\Q\in\mathcal{P}_2}\left\{\mathcal{H}(\phi,\Q)-\frac{1}{2\gamma}
            \cdot\mathcal{W}_{2}^{2}(\Q, \Q_{i})\right\}$
            \ENDFOR
        \end{algorithmic}
    \end{algorithm}

    An approximate solution to the JKO scheme can be obtained by solving
    the following parameterized transport map maximization problem, initialized with a map $T_0$ that pushes forward a continuous distribution $\nu\in\mathcal{P}_2^r$  to $\bP$:

\begin{equation}\label{eq:practical_problem}
T_{i+1}:=\argmax_{T\in L^{2}(\nu)}\left\{\mathbb{E}_{\xi \sim\nu}\left[\ell(f_{\phi}, T(\xi))-\frac{1}{2\gamma}\|T(\xi)-T_{i}(\xi)\|^{2}\right]-\lambda\cdot \mathcal{D}(T_{\#}\nu,\bP)\right\}.
\end{equation}

    \begin{prop}[Equivalent JKO step via transport maps]\label{prop:firstorder}
      Suppose that $T_{i}\in L^{2}(\nu)$ is invertible with $\nu\in\mathcal{P}_2^r$. If $T_{i+1}$ is the solution of \eqref{eq:practical_problem} and $\Q_i=(T_{i})_{\#}\nu\in\mathcal{P}_2^r$, then the pushforward $(T_{i+1})_{\#}\nu$ satisfies
        \[
        (T_{i+1})_{\#}\nu=\argmax_{\Q\in\mathcal{P}_2}\left\{\mathcal{H}(\phi,\Q)-\frac{1}{2\gamma}
            \cdot\mathcal{W}_{2}^{2}(\Q, \Q_{i})\right\},
        \]
        and also the optimality condition
        \begin{equation}\label{eq:first-order-cond}
                     \bz\in -\partial_{\mathcal{W}_2}\mathcal{H}(\phi, (T_{i+1})_{\#}\nu)\circ T_{\nu}^{(T_{i+1})_{\#}\nu}+ \frac{1}{\gamma}(T_{\nu}^{(T_{{i+1}})_{\#}\nu}-T^{i}_{(T_{{i+1}})_{\#}\nu}\circ T_{\nu}^{(T_{{i+1}})_{\#}\nu}
                ).
        \end{equation}
    \end{prop}

\begin{remark}
In the above iterative scheme, we introduce a general continuous base distribution $\nu$ instead of directly using the data distribution $\bP$, since the true data distribution is usually unknown or may be discrete. Our analysis begins with this consideration and accommodates it by assuming access to a well-behaved, known continuous distribution $\nu$.
The remaining challenge lies in identifying appropriate choices of $\nu$ that satisfy the geometric properties required by our following theoretical assumptions. 
In particular, if the reference distribution is represented by a known differentiable generator 
$G_{\theta_0}$ with base law $\nu$, i.e., $\bP=(G_{\theta_0})_{\#}\nu$, then our framework admits a generative fine-tuning interpretation that the LFD is approximated within the pushforward family $\{(G_\theta)_{\#}\nu\}$ by optimizing the generator parameters.
\end{remark}

    In the remaining part of this section, we show that  with a sufficiently large regularization
    coefficient $\lambda$ associated with common discrepancy functions over a weakly concave function $\ell(f_{\phi}, \cdot)$,
    the JKO scheme produces a sequence of transported distributions
    $\{\Q_{i}\}$ with desirable convergence properties. Specifically,
    $\{\Q_{i}\}$ achieves exponential convergence in both the Wasserstein-2
    distance $\mathcal{W}_{2}(\Q_{i}, \Q^{*})$ and the objective gap
    $\mathcal{H}(\phi, \Q^{*})-\mathcal{H}(\phi, \Q_{i})$, where $\Q^{*}$ is the
    global maximizer of $\mathcal{H}(\phi, \cdot)$. Consequently, for any desired
    accuracy in the stationarity measure, we can characterize the required number
    of subgradient calls and JKO steps in our proposed Algorithm~\ref{alg-subgdmax},
    combined with the subproblem solver Algorithm~\ref{alg-dualppm}.

    \subsection{Strong Convexity along Generalized Geodesics}

As a starting point for analyzing the convergence rate of the JKO scheme, we aim to establish exponential convergence to significantly reduce the number of costly JKO steps required in the algorithm. This is motivated by the favorable structural properties of our minimax loss function: 
\begin{equation}\label{eq:robustloss}
\mathcal{H}(\phi,\cdot)=\mathbb{E}_{\xi \sim (\cdot)}[\ell(f_{\phi}
    , \xi)]-\lambda \cdot \mathcal{D}(\cdot,\bP),
    \end{equation}
    which may exhibit desirable growth behavior similar to that of a strongly concave function.
    Strong convexity (or strong concavity in this context) plays a central role in ensuring fast convergence of iterative optimization methods. Before we proceed, we discuss how strong convexity is defined in the Wasserstein space as it is a geodesic space and does not share the
    usual linear structure. We should clarify the definition of strong/weak
    convexity along the geodesics and more general curves.

By examining the structure of \eqref{eq:robustloss}, we can observe that the mapping
$\mathbb{Q} \mapsto \mathbb{E}_{\xi \sim \mathbb{Q}}[\ell(f_{\phi}, \xi)]$
is ``weakly concave''. If the discrepancy term $\mathcal{D}(\mathbb{Q}, \mathbb{P})$ is ``strongly convex'', then for sufficiently large $\lambda$, the overall objective $\mathcal{H}(\phi, \cdot)$ becomes ``strongly concave''. However, whether such properties hold depends heavily on which curves
    we consider. For example, from \cite[Theorem~7.3.2]{ambrosio2008gradient}
    we know that when $\mathcal{D}=\frac{1}{2}\mathcal{W}_{2}^{2}$, the squared Wasserstein
    distance $\Q\mapsto \frac{1}{2}\mathcal{W}_{2}^{2}(\Q,\mathbb{P})$ is $1$-weakly concave
    (i.e., not convex) along geodesics. Moreover, in dimensions greater than one,
    \cite[Example~9.1.5]{ambrosio2008gradient} provides a counterexample showing that strong convexity does not generally hold.
        \begin{fact}
 Wasserstein   distance  $\Q \mapsto \frac{1}{2}\mathcal{W}
    _{2}^{2}(\Q,\mathbb{P})$ is not $1$-strongly convex along geodesics.
    \end{fact}

    Consequently, we need a more general notion than standard geodesics between probability
    measures. Here, we introduce convexity along generalized geodesics as mentioned
    in \eqref{eq:gene_geodesic} to address this gap. Different from the geodesics, the generalized geodesics involve a third continuous
    distribution $\nu$ and are defined using interpolation of the two OT maps from
    $\nu$ to $\mu_{1}$ and $\mu_{2}$, respectively. Specifically, let
    $\nu \in \mathcal{P}_{2}^{r}$ and $\mu_{1}, \mu_{2}\in \mathcal{P}_{2}$,
    then a generalized geodesic joining $\mu_{1}$ to $\mu_{2}$ centered at $\nu$
    is a curve of type
    \begin{equation}
        \label{eq:gene_geodesic}\mu_{t}^{1 \rightarrow 2}:=\left((1-t) T_{\nu}^{\mu_1}
        +t T_{\nu}^{\mu_2}\right)_{\#}\nu, \quad t \in[0,1] .
    \end{equation}
    This will also lead to the generalized definition of the Wasserstein distance (i.e., linearized OT distance):
    \begin{equation*}
        \mathcal{W}_{\nu}^{2}(\mu_{1}, \mu_{2}):=\int_{\mathbb{R}^d}\|T_{\nu}^{\mu_1}
        (x)-T_{\nu}^{\mu_2}(x)\|^{2}\operatorname{d \nu}(x).
    \end{equation*}

    \begin{definition}[Convexity along generalized geodesics]
        For $\lambda \geq 0$, a functional $h$ on $\mathcal{P}_{2}$ is said to
        be $\lambda$-strongly (or $-\lambda$-weakly) convex along generalized
        geodesics (a.g.g.) centered at $\nu \in \mathcal{P}_{2}^{r}$ if for any
        $\mu_{1}, \mu_{2}\in \mathcal{P}_{2}$ and $t \in[0,1]$,
        \[
            h(\mu_{t}^{1 \rightarrow 2}) \leq(1-t) h(\mu_{1})+t h(\mu_{2})-\frac{\lambda}{2}
            t(1-t) \mathcal{W}_{\nu}^{2}(\mu_{1}, \mu_{2}).
        \]
    \end{definition}

Fortunately, we can identify several important cases where strong convexity holds along certain  different generalized geodesics.

    \begin{prop}
        [Examples for a.g.g. strongly convex discrepancy functions]\label{prop:sc_wass_entropy}
        Let $\mathbb{P}\in \mathcal{P}_{2}^{r}$ with the density function $p$.
        The following functions satisfy $1$-strongly convex property with
        respect to $\Q$:
        \begin{enumerate}
            \item[{\rm (i)}] Wasserstein distance $\mathcal{D}(\Q,\mathbb{P})=\frac{1}{2}
                \mathcal{W}_{2}^{2}(\Q,\mathbb{P})$ along $((1-t) T_{\mathbb{P}}^{\Q_1}
                +t T_{\mathbb{P}}^{\Q_2})_{\#}\mathbb{P}$ with $t \in[0,1]$ and $\mathbb{Q}_1,\mathbb{Q}_2\in \mathcal{P}_{2}$;

            \item[{\rm (ii)}] Relative entropy $\mathcal{D}(\Q,\mathbb{P})=\operatorname{KL}
                (\Q,\mathbb{P})$ along all generalized geodesics, if
                $\Q \ll \mathbb{P}$ and $p\propto \exp(-G)$, where the potential
                function $G: \mathbb{R}^{d}\rightarrow(-\infty, \infty]$ is proper,
                lower semi-continuous, $1$-strongly convex, and bounded from
                below.
        \end{enumerate}
    \end{prop}

    \begin{remark}
        Proposition~\ref{prop:sc_wass_entropy}(i) demonstrates that the squared
        Wasserstein distance $\mathcal{W}_{2}^{2}(\cdot, \mathbb{P})$ is
        geodesically convex only along generalized geodesics centered at
        $\mathbb{P}$. This cannot be extended to all generalized geodesics. The
        strong convexity of $G$ in Proposition~\ref{prop:sc_wass_entropy}(ii) is
        nearly necessary for the relative entropy to exhibit a.g.g. strong
        convexity. This is because the strong convexity of $G$ implies the log-concavity
        of $\mathbb{P}$. In fact, as demonstrated in Theorem 9.4.11 and Section
        9.4.1 of \cite{ambrosio2008gradient}, log-concavity is equivalent to the
        geodesic convexity a.g.g. of the relative entropy functional.
    \end{remark}

    \subsection{Oracle Complexity}
    Leveraging the previously discussed strongly convex a.g.g. property, we are now
    prepared to derive the oracle complexity of our proposed algorithm. We begin
    by outlining the following necessary assumptions.
    \begin{assumption}
        \label{f_assump_sc} The following assumptions on the problem \eqref{Problem-r} hold:
        \begin{itemize}
            \item[{\rm (i)}] $\ell(f_{\phi}, \cdot)$ is $\rho$-weakly concave,
                upper semi-continuous, and bounded from above;
            \item[{\rm (ii)}] $\mathcal{D}(\cdot,\mathbb{P})$ is $1$-strongly
                convex a.g.g. centered at $\nu \in \mathcal{P}_{2}^{r}$ with the
                regularization parameter $\lambda> \rho$.
        \end{itemize}
    \end{assumption}

A sufficiently large regularization parameter $\lambda$, corresponding to a small ambiguity set radius, induces a favorable optimization landscape as demonstrated by the following lemma.

    \begin{lemma}[Strong concavity of $\mathcal{H}(\phi,\cdot)$]
        \label{lem:Haggcvx} Under Assumption \ref{f_assump_sc}, the function
        $\mathcal{H}(\phi,\cdot)$ is $(\lambda-\rho)$-strongly concave a.g.g. centered
        at $\nu$ in $\mathcal{P}_{2}$.
    \end{lemma}

    Since $\mathcal{H}(\phi, \cdot)$ is $(\lambda - \rho)$-strongly concave a.g.g.
    in $\mathcal{P}_{2}$ by Lemma~\ref{lem:Haggcvx}, we can derive the
    exponential convergence of the inexact proximal point method along the generalized geodesics in Wasserstein
    space, as proposed in Algorithm~\ref{alg-dualppm}, to the optimal solution
    of the subproblem in Algorithm~\ref{alg-subgdmax}.

    \begin{assumption}[Approximate $i$-th step solution]\label{ass:learningerror}
        Let $\epsilon' \geq 0$. For $i$-th step ($i=0,1,\ldots, I-1$), there exists
        \[
            \xi_{i+1}\in -\partial_{\mathcal{W}_2}\mathcal{H}(\phi, \Q_{i+1})\circ
            T_{\nu}^{i+1}+ \frac{1}{\gamma}(T_{\nu}^{i+1}-T^{i}_{i+1}\circ T_{\nu}^{i+1})
        \]
        such that $\| \xi_{i+1}\|_{\nu}\leq \epsilon'$,
        where $T_{\nu}^{i}$ is the OT map from $\nu \in \mathcal{P}_{2}^{r}$ to
        $\Q_{i}= (T_{i})_{\#}\nu$. 
    \end{assumption}

    Using the aforementioned assumption on the accuracy of the inexact JKO
    step in Algorithm~\ref{alg-dualppm} (which can also be satisfied by the parameterized transport map
maximization problem thanks to \eqref{eq:first-order-cond}), we obtain the following results.

    \begin{prop}[Convergence rate of inexact JKO scheme]
        \label{prop:dualppm}
        Suppose that Assumptions \ref{f_assump_sc} and \ref{ass:learningerror} hold.
        Then for the sequence generated by Algorithm \ref{alg-dualppm} with step size $\gamma>\frac{36}{\lambda-\rho}$, one has for $i=0,1, \ldots, I-1$  that
        \begin{equation}
            \label{eq:ppm-expcon}\mathcal{W}_{\nu}^{2}(\Q_{i},\Q^{*}) \leq\tau^{i}\cdot\mathcal{W}_{\nu}^{2}(\Q_{0}
            ,\Q^{*}) +\frac{9 \epsilon'^{2}}{(\lambda-\rho)^{2}},
        \end{equation}
        where $
\tau:=\frac{\gamma (\lambda-\rho)+36}{\gamma^2 (\lambda-\rho)^2/2+\gamma (\lambda-\rho)-36}\in(0,1)$.
        In particular, if
        \[
            i \geq \frac{2\left(\log \mathcal{W}_{\nu}(\Q_{0},\Q^{*})+\log ((\lambda-\rho)
            / \epsilon')\right)}{\log(\tau^{-1})},
        \]
        then for any $\phi\in\Phi$, it follows that $\mathcal{W}_{\nu}(\Q_{i},\Q^{*}
        ) \leq \frac{4\epsilon'}{\lambda-\rho}$ and
        \[
            \mathcal{H}(\phi,\Q^{*})-\mathcal{H}(\phi,\Q_{i+1}) \leq \frac{5\epsilon'^2}{2(\lambda-\rho)}.
        \]
    \end{prop}

    \begin{remark}
        The convergence analysis presented in Proposition \ref{prop:dualppm} is analogous
        to that of the Wasserstein proximal point method \cite{cheng2024convergence}
        and the Wasserstein proximal gradient method \cite{salim2020wasserstein}.
        The key difference lies in the requirement of strong convexity along all
        generalized geodesics in these two works, with their analyses depending on
        varying base distributions. In contrast, our approach fixes the base distribution
        $\nu$, ensuring convexity only along certain fixed curves. This distinction
        is significant because the Wasserstein distance exhibits convexity only along
        specific curves rather than universally, as discussed in Proposition
        \ref{prop:sc_wass_entropy}. Unlike the standard proximal point update scheme, which imposes no lower bound on the stepsize $\gamma$, our framework requires a lower bound to control the approximation error between geodesics and generalized geodesics. Thanks to the exponential convergence result, the error introduced by the inexact JKO steps remains controlled and does not accumulate across iterations. This ensures that the error at each step stays bounded by a fixed value $\epsilon'$, as stated in Assumption~\ref{ass:learningerror}.
    \end{remark}

    The above proposition allows us to establish nonasymptotic convergence results
    for our proposed Algorithm \ref{alg-subgdmax}, specifically regarding the
    number of subgradient evaluations and the required number of inexact
 JKO steps.

    \begin{theorem}[Oracle complexity]
        \label{thm:oraclecomplexity}
        Let $\varepsilon\ge0$. Suppose that the Assumptions \ref{f_assump}, \ref{f_assump_sc}
        and \ref{ass:learningerror} hold. Then Algorithm \ref{alg-subgdmax} with
        step size $\eta=\mathcal{O}(\varepsilon^{2})$ and equipped with subproblem
        solver Algorithm \ref{alg-dualppm} with $\gamma>\frac{36}{\lambda-\rho}$
        will return a solution $\phi^{*}$ such that
        \[
            \|\nabla \mathcal{V}_{1/2\rho}(\phi^{*})\| \leq \varepsilon
        \]
        within $\mathcal{O}(\varepsilon^{-4})$ subgradient oracle calls and $\mathcal{O}
        (\varepsilon^{-4}\log(1/\varepsilon))$ inexact JKO steps in Assumption
        \ref{ass:learningerror} with accuracy $\epsilon'=\mathcal{O}(\varepsilon)$.
    \end{theorem}

One remaining issue lies in the small step size $\eta = \mathcal{O}(\varepsilon^{2})$ required in 
Theorems~\ref{thm:subgconvergence} and \ref{thm:oraclecomplexity}, which compensates for the nonsmooth nature of the problem. In the next section, we show that under a smoothness condition, the complexity 
bounds can be further improved and the double-loop structure can be avoided.

    \section{Smoothness and Alternating Scheme}
    \label{sec:smooth}

    The nonsmoothness arises from two sources in the problem \eqref{Problem-r}:
    \begin{enumerate}
        \item[(i)] The nonsmoothness of $\ell(f_{(\cdot)}, \xi)$ for any $\xi \in
            \R^d$, which is typically from nonsmooth modeling elements, e.g.,
            regularizers and parameterizations;
        \item[(ii)] The nonunique solutions of the LFD in the maximization problem.
    \end{enumerate}
    For (i), in this section we additionally assume that $\ell(f
    _{(\cdot)}, \xi)$ is gradient Lipschitz continuous. For (ii), we impose Assumption~\ref{f_assump_sc}
    to ensure that the problem is strongly concave a.g.g., which guarantees the
    uniqueness of the LFD.

\begin{assumption}
\label{f_assump_smooth}
For problem~\eqref{Problem-r}, the loss $\ell(f_{(\cdot)},\cdot)$ is
continuously differentiable and has an $L$-Lipschitz gradient on
$\Phi\times\mathbb R^d$. The regularization parameter satisfies
$\lambda>L$. The support of $\Q^*(\phi):=\argmax_{\Q\in\mathcal{P}_2}\mathcal{H}(\phi,\Q)$ is contained in a fixed compact set $\mathcal{X}$ for any $\phi\in\Phi$. $\mathcal H$ and $\nabla_\phi\mathcal H$ are continuous
on $\Phi\times\mathcal P_2(\mathcal X)$, where $\mathcal P_2(\mathcal X)$
is equipped with the $\mathcal W_2$ topology.
\end{assumption}

    Although this assumption introduces
    an additional requirement for the model, it does not significantly restrict its
    applicability. For example, in practice, we can use smooth activation
    functions in neural networks, such as the softplus function, to replace the
    highly nonsmooth ReLU activation function. Then, with the help of the generalized Danskin's
    theorem, we derive the smoothness of the worst-case value function $\mathcal{V}$.

    \begin{prop}
        [Danskin's theorem]\label{prop:danskin} Suppose that Assumptions
        \ref{f_assump_sc} and \ref{f_assump_smooth} hold. Then the worst-case value function
        $\mathcal{V}$ is differentiable at any $\phi\in\Phi$ and
        \[
            \nabla\mathcal{V}(\phi)=\nabla_{\phi}\mathcal{H}(\phi,\Q^{*}(\phi)).
        \]
    \end{prop}

    By applying Danskin's theorem as established in Proposition~\ref{prop:danskin}, we derive key Lipschitz continuity properties of the solution mapping, which in turn imply the Lipschitz gradient of the worst-case value function $\mathcal{V}$.

    \begin{lemma}[Lipschitz solution mapping]
        \label{lem:q-lipschitz} Under Assumptions \ref{f_assump_sc} and \ref{f_assump_smooth},
        the solution mapping $\Q^{*}(\phi)$ is $\kappa_1$-Lipschitz with respect to
        $\phi\in\Phi$ and the worst-case value function $\mathcal{V}$ is
        $2\kappa_2L$-gradient Lipschitz, where $\kappa_1:=\frac{2\lambda-L}{\lambda-L}$ and $\kappa_2:=\frac{3\lambda-2L}{2(\lambda-L)}$.
    \end{lemma}

Next, we derive improved oracle complexity results under smoothness for Algorithm~\ref{alg-subgdmax}, using the subproblem solver Algorithm~\ref{alg-dualppm} introduced in the previous section.

    \begin{theorem}[Oracle complexity under smoothness]
        \label{thm:smooth_oracle_complexity} Let $\varepsilon\ge0$. Suppose that
        Assumptions \ref{f_assump_sc} and \ref{f_assump_smooth} hold. Then for the sequence generated by Algorithm \ref{alg-subgdmax} with constant step
        size $\eta_k=\frac{\lambda-L}{2L(4\lambda-3L)}$, there exists a
        $k\in\{0,1,\ldots,K-1\}$ such that
        \[
            \begin{aligned}
            \dist(0,\nabla \mathcal{V}(\phi_{k+1})+\partial\iota_{\Phi}(\phi_{k+1}))                    \leq \mathcal{O}(K^{-1/2}) + \mathcal{O}(\sqrt{\epsilon}).  
            \end{aligned}
        \]
        Suppose further that Assumption \ref{ass:learningerror} holds.
        Then Algorithm \ref{alg-subgdmax} equipped with subproblem solver Algorithm
        \ref{alg-dualppm} with $\gamma>\frac{36}{\lambda-\rho}$
        will return a solution $\phi^{*}$ such that
        \[
        \dist(0,\nabla \mathcal{V}(\phi^*)+\partial\iota_{\Phi}(\phi^*))\leq \varepsilon
        \]
        within $\mathcal{O}(\varepsilon^{-2})$ gradient oracle calls and $\mathcal{O}
        (\varepsilon^{-2}\log(1/\varepsilon))$ inexact JKO steps in Assumption
        \ref{ass:learningerror} with accuracy $\epsilon'=\mathcal{O}(\varepsilon)$.
    \end{theorem}

Furthermore, we propose the following alternating scheme for the fully smooth case, i.e., the parameter space is $\Phi = \mathbb{R}^m$. Unlike the eliminating approach in Algorithm \ref{alg-subgdmax}, it avoids the time-consuming JKO iterations in the subproblem by performing only a single update per outer iteration.

\begin{algorithm}[H]
    \caption{Alternating Algorithm}
    \label{alg:single}
    \begin{algorithmic}[1]
        \STATE \textbf{Input:} Initialization $\phi_0$ and $\Q_0$, step size $\eta > 0$, regularization parameter $\lambda > 0$, data $\xi \sim \mathbb{P}$
        \FOR{$k = 0$ {\bfseries to} $K-1$}
            \STATE $\Q_{k+1} := \argmax\limits_{\Q \in \mathcal{P}_2} \left\{ \mathcal{H}(\phi_{k}, \Q) - \frac{1}{2\gamma} \cdot \mathcal{W}_2^2(\Q, \Q_{k}) \right\}$
            \STATE $\phi_{k+1} := \phi_{k} - \eta \cdot \nabla_{\phi} \mathcal{H}(\phi_{k}, \Q_{k+1})$
        \ENDFOR
    \end{algorithmic}
\end{algorithm}

\begin{theorem}[Oracle complexity for alternating scheme]\label{thm:alternating}
Suppose that Assumptions \ref{f_assump_sc}, \ref{ass:learningerror} and \ref{f_assump_smooth} hold. 
Then for the sequence generated by Algorithm \ref{alg:single} with step size $\eta\leq\min\{\frac{1-\tau}{16\kappa_1 L},\frac{1}{32\kappa_2 L}\}$ ($\tau$ defined in Proposition \ref{prop:dualppm}) and $\gamma>\frac{36}{\lambda-\rho}$, there exists a $k\in\{0,1,\ldots,K-1\}$ such that
$$
\left\|\nabla \mathcal{V}(\phi_{k})\right\| \leq 
\mathcal{O}(K^{-1/2})+\mathcal{O}(\epsilon').
$$
Moreover, Algorithm \ref{alg:single} 
        will return a solution $\phi^{*}$ such that
        $\|\nabla \mathcal{V}(\phi^{*})\| \leq \varepsilon$
        within $\mathcal{O}(\varepsilon^{-2})$ gradient oracle calls and $\mathcal{O}
        (\varepsilon^{-2})$ inexact JKO steps in Assumption
        \ref{ass:learningerror} with accuracy $\epsilon'=\mathcal{O}(\varepsilon)$.
\end{theorem}

Together, Theorems~\ref{thm:smooth_oracle_complexity} and~\ref{thm:alternating} quantify the computational cost of the elimination-based and alternating schemes in terms of gradient oracle calls and inexact JKO oracle calls. 
In practice, each inexact JKO step can be implemented by optimizing a parameterized neural transport map in \eqref{eq:practical_problem}; the resulting oracle error $\epsilon'$ captures the combined effect of optimization, approximation, and sampling errors. 
This oracle viewpoint separates the Wasserstein-space convergence analysis from the specific neural parameterization, while retaining the computational advantage that a learned map can generate worst-case samples by a forward evaluation.

\begin{remark}
The elimination-based scheme in Theorem~\ref{thm:smooth_oracle_complexity} requires $\mathcal{O}(\varepsilon^{-2})$ gradient oracle calls and $\mathcal{O}(\varepsilon^{-2}\log(1/\varepsilon))$ inexact JKO oracle calls, because each outer iteration solves the inner distributional problem to the required accuracy. 
In contrast, the alternating scheme in Theorem~\ref{thm:alternating} requires $\mathcal{O}(\varepsilon^{-2})$ gradient oracle calls and $\mathcal{O}(\varepsilon^{-2})$ inexact JKO oracle calls by performing one JKO update per outer iteration. 
Thus, the alternating scheme removes the logarithmic inner-loop factor in the JKO oracle complexity.
In the neural transport-map implementation, both variants maintain the current learned map and generate worst-case samples by forward evaluation; the oracle counts therefore describe training-time distributional updates rather than a requirement to store all intermediate networks at inference time.
\end{remark}

    \section{Implementation of Wasserstein-2 Inner Update}\label{sec:practical-w2}

    The convergence results above use inexact JKO steps as the basic computational primitive over probability measures. To apply the algorithms, these steps must be approximated by finite-dimensional computations. The $\mathcal W_2$-regularized DRO problem is a natural setting for this approximation because it is closely related to the Wasserstein robust training formulation based on pointwise perturbations \cite{sinha2018certifying}: the least-favorable distribution update maximizes expected loss penalized by a quadratic transport cost. Because this quadratic cost can be written through transport maps, the abstract update over $\Q$ can be approximated in finite dimensions either by optimizing transported samples directly or by optimizing a parameterized map.

    At the population level, following the transport-map formulation in \eqref{eq:max-subproblem}--\eqref{eq:practical_problem}, let $\nu$ be a continuous base distribution, let $T_0{}_{\#}\nu=\bP=\Q_0$, and let $T_k{}_{\#}\nu=\Q_k$. For fixed $f_{\phi_k}$, a new distribution is represented as $T_{\#}\nu$, and the map objective for the JKO step is
    \begin{align}
        \max_T
        \mathbb{E}_{\zeta\sim\nu}
        \left[
            \ell(f_{\phi_k},T(\zeta))
            -\frac{\lambda}{2}\|T(\zeta)-T_0(\zeta)\|^2
            -\frac{1}{2\gamma}\|T(\zeta)-T_k(\zeta)\|^2
        \right].
        \label{eq:flow-ot}
    \end{align}
    Setting $\gamma=\infty$ removes the proximal term and recovers the direct inner maximization in Algorithm~\ref{alg-subgdmax}; finite $\gamma$ gives the proximal update used by the alternating scheme.

    A common implementation is the direct-sampling specialization of this formulation. When samples can be drawn directly from $\bP$, one may take $\nu=\bP$ and $T_0=I_d$. Then the base sample $\zeta$ above is a reference sample, written as $\xi_i^0\sim\bP$ in a finite batch $i = 1, 2\ldots, n$. The current map gives $\xi_i^k=T_k(\xi_i^0)$, and the empirical inner problem chooses updated samples $\xi_i$ by maximizing
    \[
        \max_{\xi_1,\ldots,\xi_n}
        \frac{1}{n}\sum_{i=1}^n
        \left[
            \ell(f_{\phi_k},\xi_i)
            -\frac{\lambda}{2}\|\xi_i-\xi_i^0\|^2
            -\frac{1}{2\gamma}\|\xi_i-\xi_i^k\|^2
        \right].
    \]
    This empirical problem leads to two implementation families. Particle methods optimize the updated samples $\xi_i$ directly and provide finite-sample baselines for the direct inner maximization, especially in the elimination case $\gamma=\infty$. Neural transport maps instead impose the shared form $\xi_i=T_\theta(\xi_i^0)$, which amortizes the update across samples and can be applied immediately to newly drawn data. This is particularly useful for the alternating scheme, where the proximal term can be evaluated on new samples by comparing $T_\theta(\xi_i^0)$ with $T_k(\xi_i^0)$.

    For supervised learning, a sample can be written as $\xi=(x,y)$. The transport may keep the label fixed and move only the feature vector, so the quadratic costs are computed in feature space. If $\xi_i^0=(x_i,y_i)$, we write the updated sample as $\xi_i=(z_i,y_i)$. Particle methods optimize the feature vectors $z_i$ directly, while neural maps can use the label-conditioned form $z_i=T_{\theta,y_i}(x_i)$.

    \section{Numerical Experiments}\label{sec:numerical}

    We now evaluate the implementation choices above on two tasks. The two-dimensional Gaussian mixture isolates the inner $\Q$-update and visualizes the transported samples, while the CIFAR-10 feature-space experiment tests the same inner maximization in a higher-dimensional representation. The comparisons focus on the effect of the JKO proximal term for neural maps and, in the CIFAR-10 experiment, on whether neural maps are better used directly or as initializers for particle L-BFGS.
The cross-entropy losses used below satisfy the assumptions in the theoretical analysis under the boundedness and smoothness conditions induced by the compact parameter/feature domains and the smooth SiLU networks. Together with the Wasserstein-2 regularizer, whose a.g.g. strong convexity is established, this verifies the additional assumptions used by the JKO and alternating schemes.
The synthetic experiment verifies that neural transport maps can qualitatively match particle-based least-favorable movements while preserving the fast inference advantage of a learned map. The CIFAR-10 experiment tests the same mechanism in a higher-dimensional setting and shows that neural maps are useful as standalone solvers, while being particularly effective as initializers for particle methods.

    \subsection{Two-Dimensional Gaussian Mixture}
    \paragraph{Setup.} We instantiate the Wasserstein-2 DRO problem on a binary two-dimensional classification task. Class 0 follows $\mathbb{P}(x|{y=0})=\mathcal{N}(0,I_2)$, and class 1 follows a 25-component Gaussian mixture $\mathbb{P}(x|{y=1})$ with means on $\{-10,-5,0,5,10\}^2$ and covariance $0.1I_2$. For $\xi=(x,y)$, $y\in\{0,1\}$, we use the binary cross-entropy $\ell(f_\phi,(x,y))=-\log\frac{\exp(f_{\phi,y}(x))}{\exp(f_{\phi,0}(x))+\exp(f_{\phi,1}(x))}$.
    For numerical illustrations, we consider the example where only one class sample is perturbed.
    The inner maximization is applied only to class-0 samples; class-1 samples remain fixed. Thus below $T_\theta$ denotes $T_{\theta,0}$. At outer iteration $k$, we draw class-0 samples $\{x_i\}_{i=1}^n$ and class-1 samples $\{x'_j\}_{j=1}^n$. We denote the moved class-0 points by $z_i$: particle methods optimize $z_i$ directly, while neural-map methods set $z_i=T_\theta(x_i)$. After the inner update, the classifier is updated using $\widehat{\mathbb{Q}}_k = \frac{1}{2n}\sum_{i=1}^{n}\delta_{(z_i,0)} + \frac{1}{2n}\sum_{j=1}^{n}\delta_{(x'_j,1)}$. The corresponding empirical minimax objective is
    \[
        \widehat{\mathcal{H}}_k
        =
        \frac{1}{2n}\sum_{i=1}^{n}\ell(f_{\phi_k},(z_i,0))
        +
        \frac{1}{2n}\sum_{j=1}^{n}\ell(f_{\phi_k},(x'_j,1))
        -
        \frac{\lambda}{4n}\sum_{i=1}^{n}\|z_i-x_i\|^2 .
    \]
    The first two terms are the empirical classifier loss.
    For the inner stationarity metric, define the particle objective $h_i(z;\phi_k):=\ell(f_{\phi_k},(z,0))-\frac{\lambda}{2}\|z-x_i\|^2$.
    For a fixed batch, the particle inner problem is $\max_{z_1,\ldots,z_n}\sum_i h_i(z_i;\phi_k)$. The full setup is given in Appendix Table~\ref{tab:gm-exp-details}.

    \paragraph{Methods.}
    We compare particle updates with neural transport maps. The particle baselines are \textbf{L-BFGS}, implemented as one batched solve over the stacked vector $(z_1,\ldots,z_n)$, and \textbf{GD}, which follows the rescaled update of \cite{sinha2018certifying} by minimizing $-\lambda^{-1}\ell(f_\phi,(z,y))+\frac{1}{2}\|z-x\|^2$ with diminishing steps and the gradient warm start from \cite{sinha2018certifying}. We also report an unrescaled GD variant in Figure~\ref{fig:gm-gd-comparison}. The neural methods use residual maps $T_\theta(x)=x+R_\theta(x)$. \textbf{Elim-$B$} performs $B\in\{5,10,15\}$ optimizer steps on
    \[
        \max_\theta
        \frac{1}{n}\sum_{i=1}^{n}
        \left[
        \ell(f_{\phi_k},(T_\theta(x_i),0))
        -
        \frac{\lambda}{2}\|T_\theta(x_i)-x_i\|^2
        \right],
    \]
    which corresponds to \eqref{eq:flow-ot} with $\gamma=\infty$. \textbf{Alt-$B$} uses the same $B$ values and adds the JKO proximal term $-\frac{1}{2\gamma}\|T_\theta(x_i)-T_{\theta_k}(x_i)\|^2$ with $\gamma=5.0$. All neural maps are warm-started from the previous outer iteration.

    \begin{figure}[H]
        \centering
        \includegraphics[width=\columnwidth]{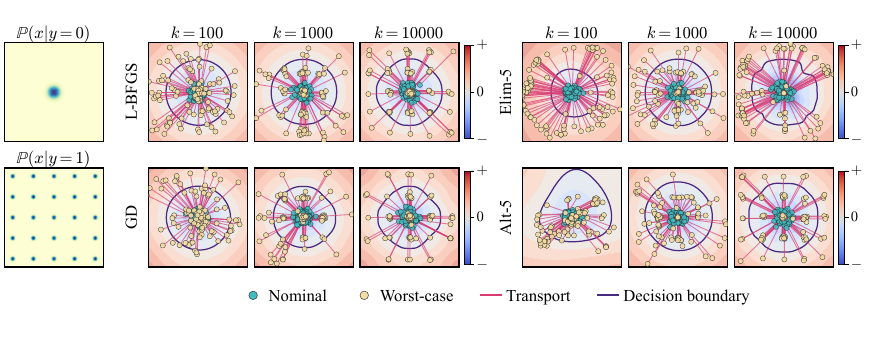}
        \caption{\textbf{Worst-case sample generation on the two-dimensional Gaussian mixture task.} The left panels show the class-0 conditional distribution $\mathbb{P}(x|{y=0})$ and the class-1 conditional mixture $\mathbb{P}(x|{y=1})$. The remaining panels compare L-BFGS, GD, Elim-5, and Alt-5 for $\lambda=2\times10^{-2}$ at outer iterations $k\in\{100,1000,10000\}$. Cyan points are original class-0 samples $x_i$, yellow points are moved samples $z_i$, and magenta segments connect each displayed pair $(x_i,z_i)$. The background color shows the classifier logit difference $f_{\phi,1}(x)-f_{\phi,0}(x)$, and the purple zero contour is the decision boundary.}\label{fig:synthetic_results}
    \end{figure}

    \paragraph{Metrics.}
    We report two convergence metrics: the classifier gradient norm $\|\nabla_\phi\widehat{\mathcal{H}}\|$, computed before gradient clipping, and the RMS particle-gradient norm
    \[
        \|\nabla_z h\|_{\mathrm{RMS}}
        :=
        \left(
        \frac{1}{n}\sum_{i=1}^{n}
        \|\nabla_z h_i(z_i;\phi_k)\|^2
        \right)^{1/2},
    \]
    which measures stationarity of the finite-particle inner problem.

    \begin{figure}[H]
        \centering
        \includegraphics[width=0.86\textwidth]{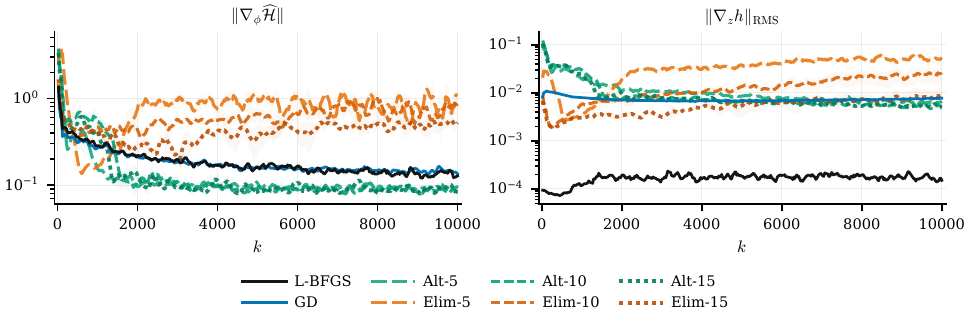}
        \caption{\textbf{Convergence curves on the two-dimensional Gaussian mixture task.} We plot the classifier gradient norm $\|\nabla_\phi\widehat{\mathcal{H}}\|$ and the RMS particle-gradient norm $\|\nabla_z h\|_{\mathrm{RMS}}$ over outer iteration $k$. Curves show means over three seeds, with shaded standard errors.}\label{fig:gm-main-convergence}
    \end{figure}

    \paragraph{Results.}
    The two-dimensional experiment supports the first claim: the JKO proximal term stabilizes the neural $\Q$-update. Figure~\ref{fig:synthetic_results} shows the moved samples: the particle methods and Alt-5 move class-0 samples toward high-loss regions near the class-1 modes, while Elim-5 is less aligned with the particle updates. Figure~\ref{fig:gm-main-convergence} shows the same pattern in the convergence metrics. With the same neural-map update budget, Alt-$B$ ends with smaller classifier gradients than Elim-$B$ ($0.085$--$0.107$ versus $0.53$--$1.61$) and smaller $\|\nabla_z h\|_{\mathrm{RMS}}$ ($5.0\times10^{-3}$--$6.7\times10^{-3}$ versus $1.0\times10^{-2}$--$4.9\times10^{-2}$). Additional loss and objective curves are given in Figure~\ref{fig:gm-objective-reference}.

    \subsection{CIFAR-10 Image Classification}
    \paragraph{Setup.} We next consider CIFAR-10 feature vectors. Images are encoded by a fixed ViT-base model, and the classifier is trained in the resulting feature space. For $\xi=(x,y)$ with $y\in\{0,\ldots,9\}$, we use the multiclass cross-entropy $\ell(f_\phi,(x,y))=-\log\frac{\exp(f_{\phi,y}(x))}{\sum_{c=0}^{9}\exp(f_{\phi,c}(x))}$.
    For a mini-batch $\{(x_i,y_i)\}_{i=1}^{n}$, the particle objective is
    \[
        \max_{z_1,\ldots,z_n}
        \frac{1}{n}\sum_{i=1}^{n}
        \left[
        \ell(f_{\phi_k},(z_i,y_i))
        -
        \frac{\lambda}{2}\|z_i-x_i\|^2
        \right].
    \]

    \paragraph{Methods.}
    We compare ERM, GD, L-BFGS, Neural, and Neural+L-BFGS. GD and L-BFGS optimize the particles, i.e., feature vectors, directly. Neural uses a label-conditioned transport map $T_{\theta,y_i}(x_i)$. Neural+L-BFGS uses this map output, $T_{\theta,y_i}(x_i)$, to initialize the L-BFGS solve. We sweep $\lambda\in\{10^{-3},10^{-2},10^{-1},1\}$ over five seeds. Robustness is evaluated by 50-step $\ell_2$ PGD in feature space with budget $0.2$ times the average feature norm. The full setup is given in Appendix Table~\ref{tab:cifar10-exp-details}.

    \paragraph{Metrics.}
    We report clean test accuracy, feature-space PGD-50 accuracy, and CIFAR-10-C \cite{hendrycks2018benchmarking} mean accuracy. The CIFAR-10-C mean averages the Weather, Blur, Noise, and Digital corruption groups. For computational cost, we report L-BFGS iterations and function evaluations. An L-BFGS iteration is one accepted quasi-Newton update of the stacked particles, while a function evaluation is one batched objective and gradient evaluation. These counts differ because line search can evaluate several candidate step lengths before accepting an update.

    \begin{figure}[H]
        \centering
        \includegraphics[width=0.72\textwidth]{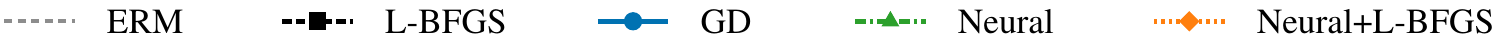}
        \vspace{0.15em}

        \begin{subfigure}{0.31\textwidth}
            \centering
            \includegraphics[width=\linewidth]{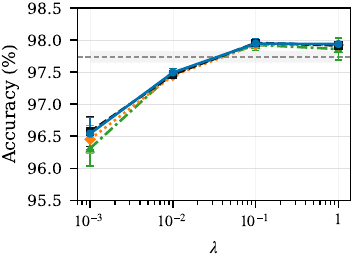}
            \caption{Clean}
        \end{subfigure}\hfill
        \begin{subfigure}{0.31\textwidth}
            \centering
            \includegraphics[width=\linewidth]{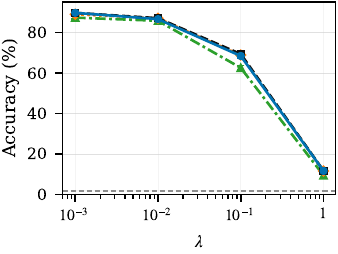}
            \caption{PGD-50}
        \end{subfigure}\hfill
        \begin{subfigure}{0.31\textwidth}
            \centering
            \includegraphics[width=\linewidth]{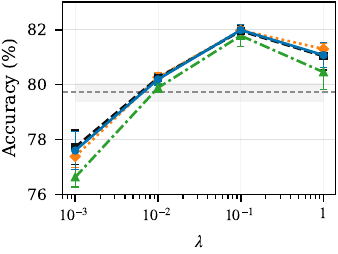}
            \caption{CIFAR-10-C}
        \end{subfigure}
        \caption{\textbf{CIFAR-10 feature-space robustness for different $\lambda$.} Markers show mean accuracy and error bars show standard deviation over five seeds. ERM is trained by empirical risk minimization on the original samples and is independent of $\lambda$.}\label{fig:cifar10-accuracy-sweep}
    \end{figure}

    \begin{figure}[H]
        \centering
        \includegraphics[width=0.72\textwidth]{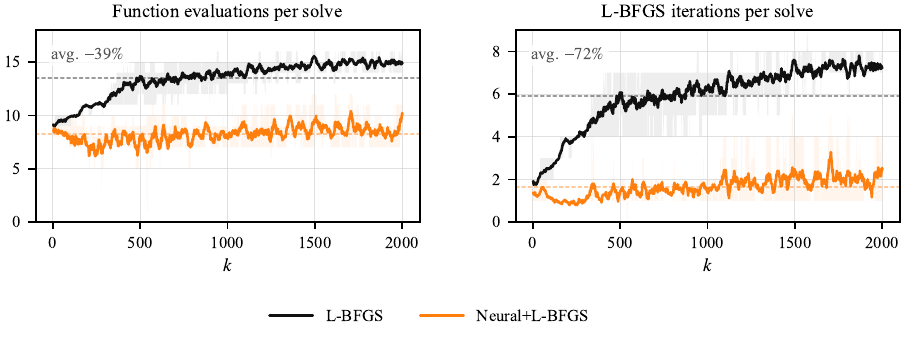}
        \caption{\textbf{Cost of L-BFGS particle optimization on CIFAR-10.} For $\lambda=10^{-3}$, curves show the number of function evaluations and accepted L-BFGS iterations per inner solve over training. Dashed horizontal lines mark averages over inner solves, and annotations report the average reduction from neural initialization.}\label{fig:cifar10-lbfgs-efficiency}
    \end{figure}

    \paragraph{Results.}
    Figure~\ref{fig:cifar10-accuracy-sweep} shows the same pattern on both robustness benchmarks. Particle GD and L-BFGS give the strongest direct inner updates, while Neural alone is consistently weaker in the high-dimensional feature space. Using the neural map as an initialization largely closes this gap: Neural+L-BFGS is within $0.7$ percentage points of GD and L-BFGS on PGD-50 and is essentially tied with them on CIFAR-10-C at the best $\lambda$ values. Clean accuracy remains near $98\%$, so the robustness improvement does not come from sacrificing nominal performance. Figure~\ref{fig:cifar10-lbfgs-efficiency} further shows that the neural initialization reduces both L-BFGS iterations and function evaluations under the same stopping criteria. Overall, these results suggest that neural maps are most useful here as amortized warm starts for particle optimization, rather than as standalone replacements for direct particle updates.

    \section{Conclusions}
    \label{sec:conclusion}
    We have investigated iterative algorithms for solving the DRO problem from the
    perspective of optimization in Wasserstein space. Unlike prior works that analyze dynamics in continuous time, we adopt a viewpoint based on transforming probability distributions via iterative
    steps. In addition to being a robust classifier, the solution derived from our algorithm
    can also be utilized to efficiently generate samples from the LFD. A future
    direction involves extending our framework to the general nonlinear DRO problems \cite{sheriff2024nonlinear}, which are receiving increasing attention due to their relevance in real-world applications.

\section*{Acknowledgment}

This work is partially supported by an NSF CAREER CCF-1650913, NSF DMS-2134037, CMMI-2015787, CMMI-2112533, DMS-1938106, DMS-1830210, ONR N000142412278, and the Coca-Cola Foundation. The authors thank Professors Xiuyuan Cheng and Johannes Milz for helpful discussions.


    \bibliography{references}

@article{kent2021modified,
  title={Modified {F}rank {W}olfe in probability space},
  author={Kent, Carson and Li, Jiajin and Blanchet, Jose and Glynn, Peter W},
  journal={Advances in Neural Information Processing Systems},
  volume={34},
  pages={14448--14462},
  year={2021}
}

@article{cheng2024convergence,
  title={Convergence of flow-based generative models via proximal gradient descent in {W}asserstein space},
  author={Cheng, Xiuyuan and Lu, Jianfeng and Tan, Yixin and Xie, Yao},
  journal={IEEE Transactions on Information Theory},
  volume={70},
  number={11},
  pages={8087--8106},
  year={2024},
  publisher={IEEE}
}

@article{xu2024flow,
  title={Flow-based distributionally robust optimization},
  author={Xu, Chen and Lee, Jonghyeok and Cheng, Xiuyuan and Xie, Yao},
  journal={IEEE Journal on Selected Areas in Information Theory},
  volume={5},
  pages={62--77},
  year={2024},
  publisher={IEEE}
}

@article{yue2022linear,
  title={On linear optimization over {W}asserstein balls},
  author={Yue, Man-Chung and Kuhn, Daniel and Wiesemann, Wolfram},
  journal={Mathematical Programming},
  volume={195},
  number={1},
  pages={1107--1122},
  year={2022},
  publisher={Springer}
}

@article{gao2023distributionally,
  title={Distributionally robust stochastic optimization with {W}asserstein distance},
  author={Gao, Rui and Kleywegt, Anton},
  journal={Mathematics of Operations Research},
  volume={48},
  number={2},
  pages={603--655},
  year={2023},
  publisher={INFORMS}
}

@article{davis2019stochastic,
  title={Stochastic model-based minimization of weakly convex functions},
  author={Davis, Damek and Drusvyatskiy, Dmitriy},
  journal={SIAM Journal on Optimization},
  volume={29},
  number={1},
  pages={207--239},
  year={2019},
  publisher={SIAM}
}

@article{mohajerin2018data,
  title={Data-driven distributionally robust optimization using the {W}asserstein metric: Performance guarantees and tractable reformulations},
  author={Mohajerin Esfahani, Peyman and Kuhn, Daniel},
  journal={Mathematical Programming},
  volume={171},
  number={1},
  pages={115--166},
  year={2018},
  publisher={Springer}
}

@article{lanzetti2025first,
  title={First-order conditions for optimization in the {W}asserstein space},
  author={Lanzetti, Nicolas and Bolognani, Saverio and D{\"o}rfler, Florian},
  journal={SIAM Journal on Mathematics of Data Science},
  volume={7},
  number={1},
  pages={274--300},
  year={2025},
  publisher={SIAM}
}

@article{zhang2024wasserstein,
  title={Wasserstein proximal operators describe score-based generative models and resolve memorization},
  author={Zhang, Benjamin J and Liu, Siting and Li, Wuchen and Katsoulakis, Markos A and Osher, Stanley J},
  journal={arXiv preprint arXiv:2402.06162},
  year={2024}
}

@article{xu2023normalizing,
  title={Normalizing flow neural networks by {JKO} scheme},
  author={Xu, Chen and Cheng, Xiuyuan and Xie, Yao},
  journal={Advances in Neural Information Processing Systems},
  volume={36},
  pages={47379--47405},
  year={2023}
}

@book{shapiro2021lectures,
  title={Lectures on Stochastic Programming: Modeling and Theory},
  author={Shapiro, Alexander and Dentcheva, Darinka and Ruszczynski, Andrzej},
  year={2021},
  publisher={SIAM}
}

@book{ambrosio2008gradient,
  title={Gradient flows: In metric spaces and in the space of probability measures},
  author={Ambrosio, Luigi and Gigli, Nicola and Savar{\'e}, Giuseppe},
  year={2008},
  publisher={Springer Science \& Business Media}
}

@article{shapiro2017distributionally,
  title={Distributionally robust stochastic programming},
  author={Shapiro, Alexander},
  journal={SIAM Journal on Optimization},
  volume={27},
  number={4},
  pages={2258--2275},
  year={2017},
  publisher={SIAM}
}

@article{bonet2024mirror,
  title={Mirror and preconditioned gradient descent in {W}asserstein space},
  author={Bonet, Cl{\'e}ment and Uscidda, Th{\'e}o and David, Adam and Aubin-Frankowski, Pierre-Cyril and Korba, Anna},
  journal={Advances in Neural Information Processing Systems},
  volume={37},
  pages={25311--25374},
  year={2024}
}

@article{delage2010distributionally,
  title={Distributionally robust optimization under moment uncertainty with application to data-driven problems},
  author={Delage, Erick and Ye, Yinyu},
  journal={Operations Research},
  volume={58},
  number={3},
  pages={595--612},
  year={2010},
  publisher={INFORMS}
}

@article{namkoong2016stochastic,
  title={Stochastic gradient methods for distributionally robust optimization with $f$-divergences},
  author={Namkoong, Hongseok and Duchi, John C},
  journal={Advances in Neural Information Processing Systems},
  volume={29},
  year={2016}
}

@article{blanchet2019quantifying,
  title={Quantifying distributional model risk via optimal transport},
  author={Blanchet, Jose and Murthy, Karthyek},
  journal={Mathematics of Operations Research},
  volume={44},
  number={2},
  pages={565--600},
  year={2019},
  publisher={INFORMS}
}

@article{kuhn2025distributionally,
  title={Distributionally robust optimization},
  author={Kuhn, Daniel and Shafiee, Soroosh and Wiesemann, Wolfram},
  journal={Acta Numerica},
  volume={34},
  pages={579--804},
  year={2025},
  publisher={Cambridge University Press}
}

@incollection{bertsimas2000moment,
  title={Moment problems and semidefinite optimization},
  author={Bertsimas, Dimitris and Sethuraman, Jay},
  booktitle={Handbook of Semidefinite Programming: Theory, Algorithms, and Applications},
  pages={469--509},
  year={2000},
  publisher={Springer}
}

@article{ben2013robust,
  title={Robust solutions of optimization problems affected by uncertain probabilities},
  author={Ben-Tal, Aharon and Den Hertog, Dick and De Waegenaere, Anja and Melenberg, Bertrand and Rennen, Gijs},
  journal={Management Science},
  volume={59},
  number={2},
  pages={341--357},
  year={2013},
  publisher={INFORMS}
}

@article{wang2016likelihood,
  title={Likelihood robust optimization for data-driven problems},
  author={Wang, Zizhuo and Glynn, Peter W and Ye, Yinyu},
  journal={Computational Management Science},
  volume={13},
  pages={241--261},
  year={2016},
  publisher={Springer}
}

@article{rahimian2022frameworks,
  title={Frameworks and results in distributionally robust optimization},
  author={Rahimian, Hamed and Mehrotra, Sanjay},
  journal={Open Journal of Mathematical Optimization},
  volume={3},
  pages={1--85},
  year={2022}
}

@article{blanchet2022optimal,
  title={Optimal transport-based distributionally robust optimization: Structural properties and iterative schemes},
  author={Blanchet, Jose and Murthy, Karthyek and Zhang, Fan},
  journal={Mathematics of Operations Research},
  volume={47},
  number={2},
  pages={1500--1529},
  year={2022},
  publisher={INFORMS}
}

@article{chen2007robust,
  title={A robust optimization perspective on stochastic programming},
  author={Chen, Xin and Sim, Melvyn and Sun, Peng},
  journal={Operations Research},
  volume={55},
  number={6},
  pages={1058--1071},
  year={2007},
  publisher={INFORMS}
}

@article{shapiro2001duality,
  title={On duality theory of conic linear problems},
  author={Shapiro, Alexander},
  journal={Nonconvex Optimization and its Applications},
  volume={57},
  pages={135--155},
  year={2001},
  publisher={Springer}
}

@article{salim2020wasserstein,
  title={The {W}asserstein proximal gradient algorithm},
  author={Salim, Adil and Korba, Anna and Luise, Giulia},
  journal={Advances in Neural Information Processing Systems},
  volume={33},
  pages={12356--12366},
  year={2020}
}

@article{wang2020information,
  title={Information {N}ewton's flow: Second-order optimization method in probability space},
  author={Wang, Yifei and Li, Wuchen},
  journal={arXiv preprint arXiv:2001.04341},
  year={2020}
}

@article{brenier1991polar,
  title={Polar factorization and monotone rearrangement of vector-valued functions},
  author={Brenier, Yann},
  journal={Communications on Pure and Applied Mathematics},
  volume={44},
  number={4},
  pages={375--417},
  year={1991},
  publisher={Wiley Online Library}
}

@article{jordan1998variational,
  title={The variational formulation of the {F}okker--{P}lanck equation},
  author={Jordan, Richard and Kinderlehrer, David and Otto, Felix},
  journal={SIAM Journal on Mathematical Analysis},
  volume={29},
  number={1},
  pages={1--17},
  year={1998},
  publisher={SIAM}
}

@article{sheriff2024nonlinear,
  title={Nonlinear distributionally robust optimization},
  author={Sheriff, Mohammed Rayyan and Mohajerin Esfahani, Peyman},
  journal={Mathematical Programming},
  pages={1--60},
  year={2024},
  publisher={Springer}
}

@book{van2000asymptotic,
  title={Asymptotic statistics},
  author={Van der Vaart, Aad W},
  volume={3},
  year={2000},
  publisher={Cambridge university press}
}

@inproceedings{sinha2018certifying,
  title={Certifying Some Distributional Robustness with Principled Adversarial Training},
  author={Sinha, Aman and Namkoong, Hongseok and Duchi, John},
  booktitle={International Conference on Learning Representations},
  year={2018}
}

@article{liu2025dro,
  title={{DRO}: A Python Library for Distributionally Robust Optimization in Machine Learning},
  author={Liu, Jiashuo and Wang, Tianyu and Lam, Henry and Namkoong, Hongseok and Blanchet, Jose},
  journal={arXiv preprint arXiv:2505.23565},
  year={2025}
}

@article{cheng2025worst,
  title={Worst-case generation via minimax optimization in {W}asserstein space},
  author={Cheng, Xiuyuan and Xie, Yao and Zhu, Linglingzhi and Zhu, Yunqin},
  journal={arXiv preprint arXiv:2512.08176},
  year={2025}
}

@inproceedings{
hendrycks2018benchmarking,
title={Benchmarking Neural Network Robustness to Common Corruptions and Perturbations},
author={Dan Hendrycks and Thomas Dietterich},
booktitle={International Conference on Learning Representations},
year={2019},
url={https://openreview.net/forum?id=HJz6tiCqYm},
}
    \bibliographystyle{alpha}

    \newpage
    \appendix

    \section{Definitions and Basic Facts}
    \label{sec:prelim}
    Let us first introduce weakly convex functions, the class of nonconvex functions studied in this
    paper.
    \begin{definition}[Weak convexity]
        \label{defi:weak-convex} A function $h:\R^{d}\rightarrow \R$ is said to be
        $\rho$-weakly convex on an arbitrary set $\mathcal{X}\subseteq \R^{d}$ for
        some constant $\rho \ge 0$ if for any $x, y \in \mathcal{X}$ and $\tau \in
        [0,1]$, we have
        \[
            h(\tau x+(1-\tau) y) \leq \tau h(x)+(1-\tau) h(y)+\frac{\rho \tau(1-\tau)}{2}
            \|x-y\|^{2}.
        \]
        The above definition is equivalent to the convexity of the function
        $h(\cdot) + \tfrac{\rho}{2}\| \cdot \|^{2}$ on $\mathcal{X}$.
    \end{definition}

    Next, we will introduce the discrepancy functions considered in the paper.
    There are two natural ways to model changes in distributions to build
    distributional ambiguity sets. The first way is to assume that the likelihood
    of the baseline model (often uniformly distributed over an observed data set)
    is corrupted. In this case, we may account for model misspecification using likelihood
    ratios. This leads to $\varphi$-divergence-based uncertainty sets. The second
    way is to consider integral probability metrics, e.g., Wasserstein distance,
    to account for potential corruptions in the baseline model by considering perturbations
    in the actual outcomes.

    \begin{definition}[Relative entropy]
        Let $\mathbb{P}, \Q$ be Borel probability measures on $\mathcal{X}$.
        Then the relative entropy of $\Q$ with respect to $\mathbb{P}$ is
        defined as
        \[
            \mathcal{D}_{\varphi}(\Q, \mathbb{P}):=
            \begin{cases}
                \int_{\mathcal{X}}\varphi\left(\frac{{\operatorname{d\Q}}}{\operatorname{d\mathbb{P}}}\right) \operatorname{d\mathbb{P}} & \text{ if }\Q \ll \mathbb{P} \\
                +\infty                                                                                                                  & \text{ otherwise. }
            \end{cases}
        \]
        where $\varphi$ is a nonnegative, lower semi-continuous, extend-valued, strictly
        convex function. Specifically, it becomes KL divergence when
        \[
            \varphi(s):=
            \begin{cases}
                s(\log s-1)+1 & \text{ if }s>0  \\
                1             & \text{ if }s=0  \\
                +\infty       & \text{ if }s<0.
            \end{cases}
        \]
    \end{definition}

    \begin{definition}[Wasserstein distance]
        The Wasserstein distance of order 2 ($\mathcal{W}_{2}$-distance) between
        two distributions in $\mathcal{P}_{2}$ is defined by
        \[
            \mathcal{W}_{2}^{2}(\bP, \Q):=\inf_{\pi \in \Pi(\bP, \Q)}\int_{\mathbb{R}^d
            \times \mathbb{R}^d}\|x-y\|^{2}\operatorname{d \pi}(x, y),
        \]
        where $\Pi(\bP, \Q)$ denotes the family of all joint couplings with $\bP,
        \Q \in \mathcal{P}_{2}$ as marginal distributions.
    \end{definition}

    When $\bP \in \mathcal{P}_{2}^{r}$, Brenier's theorem
    \cite{brenier1991polar} guarantees a well-defined and unique optimal
    transport (OT) map from $\bP$ to any $\Q \in \mathcal{P}_{2}$. Specifically,
    let $T_{\bP}^{\Q}$ denote the OT map from $\bP$ to $\Q$, which is defined $\bP$-almost
    everywhere and satisfies $(T_{\bP}^{\Q})_{\#}\bP = \Q$. For any map
    $T: \mathbb{R}^{d}\to \mathbb{R}^{d}$ such that $T_{\#}\bP = \Q$, the measure
    $(I_{d}, T)_{\#}\bP$ forms a coupling of $\bP$ and $\Q$. Consequently, we know
    that $\mathcal{W}_{2}^{2}(\Q, \bP) \leq \mathbb{E}_{\xi \sim \bP}\|\xi - T(\xi
    )\|^{2}$. Minimizing the right-hand side over all $T$ that push forward
    $\bP$ to $\Q$ is known as the Monge problem. By Brenier's theorem, when
    $\bP \in \mathcal{P}_{2}^{r}$, the OT map $T_{\bP}^{\Q}$ solves the Monge problem,
    leading to the following important identity for the Wasserstein distance in
    this case:
    \[
        \mathcal{W}_{2}^{2}(\Q, \bP) = \mathbb{E}_{\xi \sim \bP}\|\xi - T_{\bP}^{\Q}
        (\xi)\|^{2}.
    \]

    Finally, because our approach involves optimization over a probability space
    equipped with the Wasserstein metric, it is natural to introduce first-order
    concepts such as subdifferentials. Consider a proper, lower semi-continuous functional
    $\mathcal{V}: \mathcal{P}_{2}\to (-\infty, \infty]$, and let its domain be
    $\operatorname{dom}(\mathcal{V}) := \{\mu \in \mathcal{P}_{2}: \mathcal{V}(\mu
    ) < \infty\}.$
    We adopt the Fr\'echet subdifferential of $\mathcal{V}$ from \cite[Definition
    10.1.1]{ambrosio2008gradient}, focusing on its strong form. Under this definition,
    the Wasserstein subgradients provide linear approximations even if the
    perturbations do not arise from the optimal transport \cite[Proposition
    2.12]{lanzetti2025first}.

    \begin{definition}[Strong Wasserstein subdifferential]
        Given $\mathbb{P}\in \mathcal{P}_{2}$, a vector field
        $\Gamma \in L^{2}(\mathbb{P})$ is a strong Fr\'{e}chet subgradient of $\mathcal{V}$
        at $\mathbb{P}$ if for $\Delta \in L^{2}(\mathbb{P})$,
        \[
            \mathcal{V}((I_{d}+\Delta)_{\#}\mathbb{P})-\mathcal{V}(\mathbb{P}) \geq
            \langle\Gamma, \Delta\rangle_{\mathbb{P}}+o(\|\Delta\|_{\mathbb{P}})
            .
        \]
        The set $\partial_{\mathcal{W}_2}\mathcal{V}(\mathbb{P})$ including all
        strong Fr\'{e}chet subgradients is called the strong Fr\'{e}chet
        subdifferential of $\mathcal{V}$ at $\mathbb{P}$.
    \end{definition}

\section{Proofs for Section \ref{sec:alg}}

    \subsection*{Proof of Lemma \ref{lem:V-wcvx}}
        Since for each $\xi\in\R^d$ the function $\ell(f_{(\cdot)}, \xi)$
        is $\rho$-weakly convex on $\Phi$ by Assumption \ref{f_assump}, we can
        directly check by definition that the function
        $\mathbb{E}_{\xi \sim \mathbb{Q}}[\ell(f_{(\cdot)}, \xi)]=\int_{\R^d}
        \ell(f_{(\cdot)}, \xi) \operatorname{d \Q}(\xi)$
        is a $\rho$-weakly convex function on $\Phi$ for each
        $\mathbb{Q}\in \mathcal{P}_{2}$. Then we know that $\mathcal{H}(\cdot,\Q)$
        is $\rho$-weakly convex on $\Phi$, which implies for any
        $\phi,\psi\in\Phi$ and $\alpha \in[0,1]$ that
        \[
            \begin{aligned}
                       & \mathcal{H}(\alpha \phi+(1-\alpha) \psi,\Q)                                                                                             \\
                \leq\  & \alpha \mathcal{H}(\phi,\Q)+(1-\alpha) \mathcal{H}(\psi,\Q)+\frac{\rho \alpha(1-\alpha)}{2}\|\phi-\psi\|^{2}                            \\
                \leq\  & \max_{\Q'}\left\{\alpha \mathcal{H}(\phi,\Q')+(1-\alpha) \mathcal{H}(\psi,\Q')+\frac{\rho \alpha(1-\alpha)}{2}\|\phi-\psi\|^{2}\right\} \\
                \leq\  & \alpha\max_{\Q'}\mathcal{H}(\phi,\Q')+(1-\alpha) \max_{\Q'}\mathcal{H}(\psi,\Q')+\frac{\rho\alpha(1-\alpha)}{2}\|\phi-\psi\|^{2}        \\
                =\     & \alpha \mathcal{V}(\phi)+(1-\alpha) \mathcal{V}(\psi)+\frac{\rho\alpha(1-\alpha)}{2}\|\phi-\psi\|^{2},
            \end{aligned}
        \]
        where the first inequality is by the definition of the weakly convex
        function. Taking the supremum of the left hand side over $\Q$ gives the desired
        result:
        \[
            \mathcal{V}(\alpha \phi+(1-\alpha) \psi) \leq \alpha \mathcal{V}(\phi
            )+(1-\alpha) \mathcal{V}(\psi)+\frac{\rho\alpha(1-\alpha)}{2}\|\phi-\psi
            \|^{2}.
        \]
        The proof is complete.

    \subsection*{Proof of Theorem \ref{thm:subgconvergence}}
        From the proof of Lemma~\ref{lem:V-wcvx}, we know that $\mathcal{H}(\cdot
        ,\Q)$ is $\rho$-weakly convex. Furthermore, since $\Q_{k}$ is an $\epsilon$-approximate
        maximizer of $\max_{\Q \in \mathcal{P}_2}\mathcal{H}(\phi_{k},\Q)$, it
        follows for any $\phi \in \Phi$ and
        $\zeta(\phi_{k},\Q_{k})\in \partial_{\phi}\mathbb{E}_{\xi \sim \Q_k}[\ell
        (f_{(\cdot)},\xi)](\phi_{k})$
        that
        \begin{equation}
            \label{eq:max-property}
            \begin{aligned}
                \mathcal{V}(\phi) \geq \mathcal{H}(\phi, \Q_{k}) & =\mathbb{E}_{\xi \sim \mathbb{Q}_k}[\ell(f_{\phi}, \xi)]-\lambda\cdot \mathcal{D}(\Q_{k},\bP)                                                                                                             \\
                                                                 & \geq \mathbb{E}_{\xi \sim \mathbb{Q}_k}[\ell(f_{\phi_k}, \xi)+\langle\partial_{\phi}\ell(f_{\phi_k}, \xi), \phi-\phi_{k}\rangle]-\frac{\rho}{2}\|\phi-\phi_{k}\|^{2}-\lambda\cdot \mathcal{D}(\Q_{k},\bP) \\
                                                                 & = \mathcal{H}(\phi_{k}, \Q_{k})+\langle\zeta(\phi_{k},\Q_{k}), \phi-\phi_{k}\rangle-\frac{\rho}{2}\|\phi-\phi_{k}\|^{2}                                                                                   \\
                                                                 & \geq \mathcal{V}(\phi_{k})-\epsilon+\langle\zeta(\phi_{k},\Q_{k}), \phi-\phi_{k}\rangle-\frac{\rho}{2}\|\phi-\phi_{k}\|^{2},
            \end{aligned}
        \end{equation}
        where the first inequality is from the $\rho$-weakly convexity of the
        function $\ell(f_{(\cdot)}, \xi)$. On the other hand, we know from the iterate of Algorithm
        \ref{alg-subgdmax} that
        \begin{equation}
            \label{eq:sub-key00}
            \begin{aligned}
                \|\phi_{k+1}-\prox_{\mathcal{V}/2\rho}(\phi_{k})\|^{2} & = \|\proj_{\Phi}(\phi_{k}-\eta_k\cdot\zeta(\phi_{k},\Q_{k}))-\prox_{\mathcal{V}/2\rho}(\phi_{k})\|^{2}                                                              \\
                & \leq \|\phi_{k}-\eta_k\cdot\zeta(\phi_{k},\Q_{k})-\prox_{\mathcal{V}/2\rho}(\phi_{k})\|^{2}                                                                         \\
                & \leq \|\phi_{k}-\prox_{\mathcal{V}/2\rho}(\phi_{k})\|^{2}+2 \eta_k\langle\zeta(\phi_{k},\Q_{k}), \prox_{\mathcal{V}/2\rho}(\phi_{k})-\phi_{k}\rangle+\eta_k^{2}L^{2}, 
            \end{aligned}
        \end{equation}
        where the first inequality is from the nonexpansiveness of projection
        onto convex sets, and the second one is from $\|\zeta(\phi_{k},\Q_{k})\|\leq
        L$ by $L$-Lipschitz continuity of $\ell(f_{(\cdot)}, \xi)$ on $\Phi$ for
        any $\xi\in\R^d$. Combining \eqref{eq:sub-key00} and \eqref{eq:max-property}
        with $\phi=\prox_{\mathcal{V}/2\rho}(\phi_{k})$ indicates that
        \begin{equation}
            \label{eq:sub-key1}
            \begin{aligned}
                \mathcal{V}_{1/2\rho}(\phi_{k+1}) & \leq \mathcal{V}(\prox_{\mathcal{V}/2\rho}(\phi_{k}))+\rho\|\phi_{k+1}-\prox_{\mathcal{V}/2\rho}(\phi_{k})\|^{2}    \\
                                                  & \leq \mathcal{V}_{1/2\rho}(\phi_{k})+2 \eta_k \rho\langle\zeta(\phi_{k},\Q_{k}), \prox_{\mathcal{V}/2\rho}(\phi_{k})-\phi_{k}\rangle+\eta_k^{2}\rho L^{2}      \\
                                                  & \leq \mathcal{V}_{1/2\rho}(\phi_{k})+2 \eta_k \rho\left(\mathcal{V}(\prox_{\mathcal{V}/2\rho}(\phi_{k}))-\mathcal{V}(\phi_{k})+\epsilon+\frac{\rho}{2}\|\phi_{k}-\prox_{\mathcal{V}/2\rho}(\phi_{k})\|^{2}\right)+\eta_k^{2}\rho L^{2},
            \end{aligned}
        \end{equation}
        Taking the sum over $k=0,\ldots,K-1$ of \eqref{eq:sub-key1} we obtain
        \begin{align*}
              \mathcal{V}_{1/2\rho}(\phi_{K}) -\mathcal{V}_{1/2\rho}(\phi_{0}) 
             \leq\ &2  \rho \sum_{k=0}^{K-1}\eta_k\left(\mathcal{V}(\prox_{\mathcal{V}/2\rho}(\phi_{k}))-\mathcal{V}(\phi_{k})+\epsilon+\frac{\rho}{2}\|\phi_{k}-\prox_{\mathcal{V}/2\rho}(\phi_{k})\|^{2}\right)\\
             &+\rho L^{2}\sum_{k=0}^{K-1}\eta_k^2,
        \end{align*}
        which implies that
        \begin{equation}
            \label{eq:key-convergence}
            \begin{aligned}
            &\frac{1}{\sum_{k=0}^{K-1}\eta_k}\sum_{k=0}^{K-1}\eta_k\left(\mathcal{V}
            (\phi_{k})-\mathcal{V}(\prox_{\mathcal{V}/2\rho}(\phi_{k}))-\frac{\rho}{2}
            \|\phi_{k}-\prox_{\mathcal{V}/2\rho}(\phi_{k})\|^{2}\right) \\
            \leq\ & \frac{\mathcal{V}_{1/2\rho}(\phi_{0})-\min_{\phi\in\Phi}\mathcal{V}(\phi)}{2
             \rho \sum_{k=0}^{K-1}\eta_k}+\frac{ L^{2}\sum_{k=0}^{K-1}\eta_k^2}{2\sum_{k=0}^{K-1}\eta_k}+\epsilon.
             \end{aligned}
        \end{equation}
        Moreover, since $\mathcal{V}(\cdot)+\rho\|\cdot-\phi_{k}\|^{2}$ is
        $\rho$-strongly convex, we have
        \[
            \begin{aligned}
                       & \mathcal{V}(\phi_{k})-\mathcal{V}(\prox_{\mathcal{V}/2\rho}(\phi_{k}))-\frac{\rho}{2}\|\phi_{k}-\prox_{\mathcal{V}/2\rho}(\phi_{k})\|^{2}                                                                                        \\
                =\     & \mathcal{V}(\phi_{k})+\rho\|\phi_{k}-\phi_{k}\|^{2}-\mathcal{V}(\prox_{\mathcal{V}/2\rho}(\phi_{k}))-\rho\|\prox_{\mathcal{V}/2\rho}(\phi_{k})-\phi_{k}\|^{2}+\frac{\rho}{2}\|\phi_{k}-\prox_{\mathcal{V}/2\rho}(\phi_{k})\|^{2} \\
                =\     & \mathcal{V}(\phi_{k})+\rho\|\phi_{k}-\phi_{k}\|^{2}-\min_{\phi\in\Phi}\left\{ \mathcal{V}(\phi)+\rho\|\phi-\phi_{k}\|^{2}\right\}+\frac{\rho}{2}\|\phi_{k}-\prox_{\mathcal{V}/2\rho}(\phi_{k})\|^{2}                             \\
                \geq\  & \rho\|\phi_{k}-\prox_{\mathcal{V}/2\rho}(\phi_{k})\|^{2}=\frac{1}{4 \rho}\|\nabla \mathcal{V}_{1/2\rho}(\phi_{k})\|^{2},
            \end{aligned}
        \]
        where the last equality is from the explicit form of the gradient of the
        Moreau envelope \cite[Lemma 2.2]{davis2019stochastic} that $\nabla \mathcal{V}_{1/2\rho}
        (\phi)=2\rho(\phi-\prox_{\frac{1}{2\rho}\mathcal{V}+\iota_{\Phi}}(\phi))$.
        Plugging this in \eqref{eq:key-convergence}, it implies that there
        exists a $k'\in\{0,1,\ldots,K-1\}$ such that
        \[
            \begin{aligned}
                \|\nabla \mathcal{V}_{1/2\rho}(\phi_{k'})\|^{2}
                &\leq\frac{1}{\sum_{k=0}^{K-1}\eta_k}\sum_{k=0}^{K-1}\eta_k\|\nabla \mathcal{V}_{1/2\rho}(\phi_{k})\|^{2} \\
                & \leq \frac{4\rho}{\sum_{k=0}^{K-1}\eta_k}\sum_{k=0}^{K-1}\eta_k\left(\mathcal{V}(\phi_{k})-\mathcal{V}(\prox_{\mathcal{V}/2\rho}(\phi_{k}))-\frac{\rho}{2}\|\phi_{k}-\prox_{\mathcal{V}/2\rho}(\phi_{k})\|^{2}\right)                                           \\
                                                               & \leq \frac{2(\mathcal{V}_{1/2\rho}(\phi_0)-\min_{\phi\in\Phi} \mathcal{V}(\phi))+2 \rho L^{2}\sum_{k=0}^{K-1}\eta_k^2}{ \sum_{k=0}^{K-1}\eta_k}+4\rho\epsilon\\
                                                               &= \frac{2(\mathcal{V}_{1/2\rho}(\phi_0)-\min_\phi \mathcal{V}(\phi)+\rho L^2) }{\sqrt{K}}+4 \rho \epsilon,
            \end{aligned}
        \]
        where the equality is derived from $\eta_k=1 / \sqrt{K}$. The proof is complete.

\section{Proofs for Section \ref{sec:LFD}}

\subsection*{Proof of Proposition \ref{prop:firstorder}}
Denote $F_{i+1}$ the mapping such that $T_{i+1}= F_{i+1}\circ T_{i}$ since $T_{i}$ is non-degenerate. Then we know that 
    \begin{align*}
       &\max_{T\in L^{2}(\nu)}\left\{\mathbb{E}_{\xi \sim
            \nu}\left[\ell(f_{\phi}, T(\xi))-\frac{1}{2\gamma}\|T
            (\xi)-T_{i}(\xi)\|^{2}\right]-\lambda\cdot \mathcal{D}(T
            _{\#}\nu,\bP)\right\}\\
       =\ &\max_{F\in L^2(\Q_i)}\Bigg\{ \mathbb{E}_{\xi \sim \nu}\left[  \ell(f_{\phi}, F\circ T_{i}(\xi))\right] -\lambda\cdot \mathcal{D}((F\circ T_{i}
            )_{\#}\nu,\bP)- \frac{1}{2\gamma}\mathbb{E}_{\xi \sim \nu}\left[\|F\circ T_{i}(\xi) - T_{i}(\xi)\|^{2}\right] \Bigg\}\\ 
             =\ &\max_{F\in L^2(\Q_i)}\left\{\mathcal{H}(\phi,(F\circ T_{i})_{\#}\nu)-\frac{1}{2\gamma}\mathbb{E}_{\xi \sim \nu}\left[\|(F-I)\circ T_i(\xi)\|^{2}\right]\right\}\\
            =\ &\max_{F\in L^2(\Q_i)}\left\{\mathcal{H}(\phi,F_{\#}((T_{i})_{\#}\nu))-\frac{1}{2\gamma}\mathbb{E}_{\xi \sim (T_{i})_{\#}\nu}\left[\|F(\xi)-\xi\|^{2}\right]\right\}\\
             =\ &\max_{F\in L^2(\Q_i)}\left\{\mathcal{H}(\phi,F_{\#}\Q_i)-\frac{1}{2\gamma}\mathbb{E}_{\xi \sim \Q_i}\left[\|F(\xi)-\xi\|^{2}\right]\right\}\\
             =\ &\max_{\Q\in\mathcal{P}_2}\left\{\mathcal{H}(\phi,\Q)-\frac{1}{2\gamma}
            \cdot\mathcal{W}_{2}^{2}(\Q, \Q_{i})\right\},
    \end{align*}
where the third equality is from the chain rule for pushforward map and the last one is from Brenier's theorem (cf. \cite[Proposition 1]{xu2024flow}).

Next, from the optimality condition of the problem $\max_{\Q\in\mathcal{P}_2}\{\mathcal{H}(\phi,\Q)-\frac{1}{2\gamma}\cdot\mathcal{W}_{2}^{2}(\Q, \Q_{i})\}$ with the calculus rule in \cite[Proposition 2.15]{lanzetti2025first}, we know that 
\[
                0\in -\partial_{\mathcal{W}_2}\mathcal{H}(\phi, (T_{{i+1}})_{\#}
                \nu))+ \frac{1}{\gamma}(I- T^{i}_{(T_{{i+1}})_{\#}\nu}).
                \]
                Then we have
\begin{align*}
0
&\in -\partial_{\mathcal{W}_2}\mathcal{H}(\phi, (T_{{i+1}})_{\#}\nu)+ \frac{1}{\gamma}(I- T^{i}_{(T_{{i+1}})_{\#}\nu})\\
&= -\partial_{\mathcal{W}_2}\mathcal{H}(\phi, (T_{{i+1}})_{\#}\nu)\circ T_{\nu}^{(T_{{i+1}})_{\#}\nu}+ \frac{1}{\gamma}(T_{\nu}^{(T_{{i+1}})_{\#}\nu}-T^{i}_{(T_{{i+1}})_{\#}\nu}\circ T_{\nu}^{(T_{{i+1}})_{\#}\nu}
                ).
            \end{align*} 
        The proof is complete.

    \subsection*{Proof of Proposition \ref{prop:sc_wass_entropy}}
        The first part about the $1$-strong convexity a.g.g. of the function
        $\frac{1}{2}\mathcal{W}_{2}^{2}(\cdot,\bP)$ is directly from \cite[Corollary
        2.19(v)]{lanzetti2025first} and \cite[Lemma 9.2.7]{ambrosio2008gradient}.
        For the proof of the second part, we know from $\Q \ll \mathbb{P}$ that $\Q$
        has the density $q$ and
        \begin{equation}
            \label{eq:KL_equi}\operatorname{KL}(\Q,\mathbb{P})=\int_{\R^d}
            \frac{\operatorname{d\Q}}{\operatorname{d\mathbb{P}}}\log \left(\frac{\operatorname{d\Q}}{\operatorname{d\mathbb{P}}}
            \right) \operatorname{d\mathbb{P}}=\int_{\R^d}(q\log (q)-q\log
            (p))=c+\int_{\R^d}(q\log (q)+G q),
        \end{equation}
        where $c$ is a constant and the last equality is by $p\propto \exp(-G)$.

        Since $s \mapsto \log(s^{-d})= s^{d}F(s^{-d})$ with $F(s)= s \log s$ is convex
        and non-increasing on $(0,+\infty)$, we know from \cite[Proposition 9.3.9]{ambrosio2008gradient} that $q\mapsto
        \int_{\R^d}F(q)=\int_{\R^d}q\log (q)$ is convex a.g.g. in
        $\mathcal{P}_{2}^{r}$. On the other hand, since $G$ is $1$-strongly
        convex and also proper, lower semi-continuous, bounded from below, by \cite[Proposition
        9.3.2(i)]{ambrosio2008gradient} it follows that $q\mapsto\int_{\R^d}
        G q$ is $1$-convex along any interpolation curve, which implies $1$-strong
        convexity a.g.g. in $\mathcal{P}_{2}^{r}$. Hence, taking these into \eqref{eq:KL_equi}
        we have the conclusion that $\operatorname{KL}(\cdot,\mathbb{P})$ is $1$-convex
        a.g.g. in $\mathcal{P}_{2}^{r}$.

    \subsection*{Proof of Lemma \ref{lem:Haggcvx}}
        By Assumption \ref{f_assump_sc} we know that $\mathbb{E}_{\xi \sim \Q}\left
        [\ell(f_{\phi}, \xi)\right]$ is $\rho$-weakly concave a.g.g. along all generalized geodesics with
        respect to $\Q$ from \cite[Proposition 9.3.2(i)]{ambrosio2008gradient}.
        This together with the fact that $\lambda\mathcal{D}(\cdot,\mathbb{P})$ is
        $\lambda$-strongly convex a.g.g. centered at $\nu$ with the regularization parameter
        $\lambda> \rho$ implies that the function
        $\mathcal{H}(\phi,\cdot)=\mathbb{E}_{\xi \sim (\cdot)}\left[\ell(f_{\phi}
        , \xi)\right]-\lambda\mathcal{D}(\cdot,\bP)$
        is $(\lambda-\rho)$-strongly concave a.g.g. centered
        at $\nu$ in $\mathcal{P}_{2}$ by definition.
        The proof is complete.

    \subsection*{Proof of Proposition \ref{prop:dualppm}}
        The OT map from $\nu \in \mathcal{P}_{2}^{r}$ to $\Q_{i}$ is uniquely defined $\nu$-a.e. by Brenier's theorem. From the calculus rule
        of the Wasserstein subdifferential
        with the assumption on $\xi_{i+1}$ by Assumption \ref{ass:learningerror},
        we have for each $i$ that there exists
        $\eta_{i+1}\in \partial_{\mathcal{W}_2}(-\mathcal{H}(\phi,\cdot))(\Q_{i+1}
        )$
        such that
        \begin{equation}
            \label{eq:sub_bridge}\gamma (\xi_{i+1}- \eta_{i+1}\circ T_{\nu}^{i+1})+T^{i}_{i+1}\circ T_{\nu}^{i+1}-T_{\nu}^i
            =T_{\nu}^{i+1}-T_{\nu}^{i}, \quad \nu \text{-a.e. }
        \end{equation}
        Let $T_{\nu}^{*}$ denote the unique OT map from $\nu$ to $\Q^{*}$. Then
        \[
            \begin{aligned}
                \|T_{\nu}^{i+1}-T_{\nu}^{*}\|_{\nu}^{2} & =\|T_{\nu}^{i+1}-T_{\nu}^{i}+T_{\nu}^{i}-T_{\nu}^{*}\|_{\nu}^{2}                                                                                        \\
                                                        & =\|T_{\nu}^{i}-T_{\nu}^{*}\|_{\nu}^{2}+2\langle T_{\nu}^{i}-T_{\nu}^{*}, T_{\nu}^{i+1}-T_{\nu}^{i}\rangle_{\nu}+\|T_{\nu}^{i+1}-T_{\nu}^{i}\|_{\nu}^{2} \\
                                                        & \leq\|T_{\nu}^{i}-T_{\nu}^{*}\|_{\nu}^{2}+2\langle T_{\nu}^{i+1}-T_{\nu}^{*}, T_{\nu}^{i+1}-T_{\nu}^{i}\rangle_{\nu},
            \end{aligned}
        \]
        where the last inequality is from $\|T_{\nu}^{i+1}-T_{\nu}^{i}\|_{\nu}^{2}
        \geq 0$.
        This together with \eqref{eq:sub_bridge} implies that
        \begin{equation}
            \label{eq:sub_key_ineq}
            \begin{aligned}
                 & \|T_{\nu}^{i+1}-T_{\nu}^{*}\|_{\nu}^{2}\leq \|T_{\nu}^{i}-T_{\nu}^{*}\|_{\nu}^{2}+2\gamma\left\langle T_{\nu}^{i+1}-T_{\nu}^{*}, \xi_{i+1}- \eta_{i+1}\circ T_{\nu}^{i+1}+\frac{1}{\gamma}(T^{i}_{i+1}\circ T_{\nu}^{i+1}-T_{\nu}^i)\right\rangle_{\nu}.
            \end{aligned}
        \end{equation}
        On the other hand, since from Assumption \ref{f_assump_sc} we know that
        $\lambda\mathcal{D}(\cdot,\mathbb{P})$ is $\lambda$-strongly convex a.g.g.
        centered at $\nu$ with the regularization parameter $\lambda> \rho$, it
        follows from the proof of \cite[Lemma 4.1]{cheng2024convergence} (with $\mathcal{W}_{\nu}$
        replacing $\mathcal{W}_{2}$) that
        \begin{equation}
            \label{eq:sub_sc_key}
            \begin{aligned}
            & -\mathcal{H}(\phi,\Q^{*})+\mathcal{H}(\phi,\Q_{i+1}) 
                \geq\langle T_{\nu}^{*}-T_{\nu}^{i+1}, \eta_{i+1}\circ T_{\nu}^{i+1}\rangle_{\nu}+\frac{\lambda-\rho}{2}\cdot\mathcal{W}_{\nu}^{2}(\Q_{i+1}, \Q^{*}).
            \end{aligned}
        \end{equation}
        Meanwhile, by Cauchy--Schwarz inequality,
        \begin{equation}
            \label{eq:sub_innerproduct}
            \begin{aligned}
                &\left|\left\langle T_{\nu}^{*}-T_{\nu}^{i+1}, \xi_{i+1}+\frac{1}{\gamma}(T^{i}_{i+1}\circ T_{\nu}^{i+1}-T_{\nu}^i)\right\rangle_{\nu}\right|\\
                 \leq\ &\|T_{\nu}^{*}-T_{\nu}^{i+1}\|_{\nu}\cdot\left\|\xi_{i+1}+\frac{1}{\gamma}(T^{i}_{i+1}\circ T_{\nu}^{i+1}-T_{\nu}^i)\right\|_{\nu}                                                                                       \\
                                                                            \leq\ & \|T_{\nu}^{*}-T_{\nu}^{i+1}\|_{\nu}\cdot \left(\epsilon'+\frac{1}{\gamma}\|T^{i}_{i+1}\circ T_{\nu}^{i+1}-T_{\nu}^i\|_{\nu}\right)\\
                                                                            \leq\ &\frac{\lambda-\rho}{4}\|T_{\nu}^{*}-T_{\nu}^{i+1}\|_{\nu}^{2}+\frac{2}{\lambda-\rho}\left(\epsilon'^2+\frac{1}{\gamma^2}\|T^{i}_{i+1}\circ T_{\nu}^{i+1}-T_{\nu}^i\|_{\nu}^2\right),
            \end{aligned}
        \end{equation}
        where the second inequality is by
        $\|\xi_{i+1}\|_{\nu}\leq \epsilon'$ from Assumption
        \ref{ass:learningerror}.
        On the other hand, we know that
\begin{equation}
\label{eq:threepoint}
\begin{aligned}
\frac{1}{\gamma^2}\|T^{i}_{i+1}\circ T_{\nu}^{i+1}-T^{i}_{\nu}\|_\nu^2
&= \frac{1}{\gamma^2}\|(T^{i}_{i+1}-I)\circ T_{\nu}^{i+1}+T_{\nu}^{i+1}-T_{\nu}^*+T_{\nu}^*-T^{i}_{\nu}\|_\nu^2\\
&\leq \frac{3}{\gamma^2}(\mathcal{W}_2^2(\Q_i,\Q_{i+1})+\|T_{\nu}^{i+1}-T_{\nu}^*\|_{\nu}^2+\|T_{\nu}^{i}-T_{\nu}^*\|_{\nu}^2)\\
&\leq \frac{9}{\gamma^2}(\|T_{\nu}^{i+1}-T_{\nu}^*\|_{\nu}^2+\|T_{\nu}^{i}-T_{\nu}^*\|_{\nu}^2),
\end{aligned}
\end{equation}
where the last inequality is from 
\[
\mathcal{W}_2^2(\Q_i,\Q_{i+1})\leq 2\mathcal{W}_2^2(\Q_{i+1},\Q^*)+2\mathcal{W}_2^2(\Q_{i},\Q^*)\leq 2(\|T_{\nu}^{i+1}-T_{\nu}^*\|_{\nu}^2+\|T_{\nu}^{i}-T_{\nu}^*\|_{\nu}^2).
\]
Inserting \eqref{eq:sub_sc_key}, \eqref{eq:sub_innerproduct} and \eqref{eq:threepoint} into \eqref{eq:sub_key_ineq} gives
        \begin{equation}
            \label{eq:ppm_EVI}
            \begin{aligned}
                       & \left(1-\frac{\gamma (\lambda-\rho)}{2}-\frac{36}{\gamma (\lambda-\rho)}\right)\|T_{\nu}^{i+1}-T_{\nu}^{*}\|_{\nu}^{2}                                                                                                                              \\
                \leq\  &\left(1+\frac{36}{\gamma (\lambda-\rho)}\right) \|T_{\nu}^{i}-T_{\nu}^{*}\|_{\nu}^{2}+2 \gamma\left(-\mathcal{H}(\phi,\Q^{*})+\mathcal{H}(\phi,\Q_{i+1})-\frac{\lambda-\rho}{2}\cdot W_{\nu}^{2}(\Q_{i+1}, \Q^{*})\right)+\frac{4 \gamma\epsilon'^2 }{\lambda-\rho},
            \end{aligned}
        \end{equation}
        which implies that
        \[
            \begin{aligned}
                \left(1+\frac{\gamma (\lambda-\rho)}{2}-\frac{36}{\gamma (\lambda-\rho)}\right)\|T_{\nu}^{i+1}-T_{\nu}^{*}\|_{\nu}^{2}\leq \left(1+\frac{36}{\gamma (\lambda-\rho)}\right)\|T_{\nu}^{i}-T_{\nu}^{*}\|_{\nu}^{2}+\frac{4 \gamma\epsilon'^2}{\lambda-\rho}.
            \end{aligned}
        \]
Since $\gamma(\lambda-\rho)> 36$, we know that
\[
\tau:=\frac{\gamma (\lambda-\rho)+36}{\gamma^2 (\lambda-\rho)^2/2+\gamma (\lambda-\rho)-36}=\frac{1+\frac{36}{\gamma (\lambda-\rho)}}{1+\frac{\gamma (\lambda-\rho)}{2}-\frac{36}{\gamma (\lambda-\rho)}}<1
\]
By recursively applying the procedure from $i = 0$ to $i = I-1$, we obtain
        that
        \begin{equation*}
            \label{eq:jko_linear}
            \begin{aligned}
                \|T_{\nu}^{I}-T_{\nu}^{*}\|_{\nu}^{2} & \leq \tau^{I}\cdot\|T_{\nu}^{0}-T_{\nu}^{*}\|_{\nu}^{2}+\frac{4 \gamma\epsilon'^2}{\lambda-\rho}\cdot\frac{\left(1+\frac{\gamma (\lambda-\rho)}{2}-\frac{36}{\gamma (\lambda-\rho)}\right)^{-1}}{1-\tau} \\
                                                      & = \tau^{I}\cdot\|T_{\nu}^{0}-T_{\nu}^{*}\|_{\nu}^{2}+\frac{8\epsilon'^2}{(\lambda-\rho)^2-\frac{144}{\gamma^2 }}\\
                                                      & \leq \tau^{I}\cdot\|T_{\nu}^{0}-T_{\nu}^{*}\|_{\nu}^{2}+\frac{9\epsilon'^2}{(\lambda-\rho)^2},
            \end{aligned}
        \end{equation*}
        which is the desired result \eqref{eq:ppm-expcon}. From equation~\eqref{eq:ppm-expcon}, it follows that to ensure
        \[
            \|T_{\nu}^{i}- T_{\nu}^{*}\|_{\nu}^{2}\leq \frac{10\epsilon'^{2}}{(\lambda
            - \rho)^{2}},
        \]
        we require
        \begin{equation*}
            i \geq \frac{2}{\log(\tau^{-1})}\left( \log \mathcal{W}
            _{\nu}(\Q_{0}, \Q^{*}) + \log \left( \frac{\lambda - \rho}{\epsilon'}
            \right) \right).
        \end{equation*}
        Additionally, from \eqref{eq:ppm_EVI} and the condition $\|T_{\nu}^{i}- T
        _{\nu}^{*}\|_{\nu}^{2}\leq \frac{10\epsilon'^{2}}{(\lambda - \rho)^{2}}$ and $\gamma(\lambda-\rho) > 36$,
        we have
        \begin{equation*}
            \begin{aligned}
                  2 \gamma (\mathcal{H}(\phi,\Q^{*})-\mathcal{H}(\phi,\Q_{i+1}))
                  &\leq \left(1+\frac{36}{\gamma (\lambda-\rho)}\right)\|T_{\nu}^{i}-T_{\nu}^{*}\|_{\nu}^{2}+\frac{4 \gamma\epsilon'^2}{\lambda-\rho}\\
                  &<2\|T_{\nu}^{i}-T_{\nu}^{*}\|_{\nu}^{2}+\frac{4 \gamma\epsilon'^2}{\lambda-\rho}\\
                  &\leq \frac{20\epsilon'^{2}}{(\lambda - \rho)^{2}}+\frac{4 \gamma\epsilon'^2}{\lambda-\rho}\leq \frac{5\gamma\epsilon'^2}{\lambda-\rho}.
            \end{aligned}
        \end{equation*}
        The proof is complete.        

    \subsection*{Proof of Theorem \ref{thm:oraclecomplexity}}
        With the help of Theorem~\ref{thm:subgconvergence} and $\eta=\mathcal{O}(\varepsilon^{2})$, it suffices to obtain
        \begin{equation}
            \label{eq:complexityrequire}\|\nabla \mathcal{V}_{1/2\rho}(\phi_{k})\|
            ^{2}\lesssim \frac{2(\mathcal{V}_{1/2\rho}(\phi_{0})-\min_{\phi}\mathcal{V}(\phi)+\rho
            L^{2}K\varepsilon^4)}{K\varepsilon^2}+4 \rho \epsilon\leq\varepsilon^{2}.
        \end{equation}
        We divide this requirement into two components. For the first term
        \[
            \frac{2(\mathcal{V}_{1/2\rho}(\phi_{0})-\min_{\phi}\mathcal{V}(\phi)+\rho
            L^{2}K\varepsilon^4)}{K\varepsilon^2}\leq\frac{\varepsilon^{2}}{2},
        \]
        it leads to $K= \Omega(\varepsilon^{-4})$,
        which provides the lower bound for the number of outer iterations
        required by Algorithm~\ref{alg-subgdmax}.

        On the other hand, for the second term in \eqref{eq:complexityrequire}, which
        concerns the inner JKO steps by Algorithm \ref{alg-dualppm} that
        approximately solve the subproblem, Proposition~\ref{prop:dualppm}
        indicates that if
        $I-1 \geq \frac{2}{\log(\tau^{-1})}\left(\log \mathcal{W}
        _{\nu}(\Q_{0},\Q^{*})+\log ((\lambda-\rho) / \epsilon')\right)$, then
        \[
            \mathcal{H}(\phi, \Q^{*})-\mathcal{H}(\phi, \Q_{I})  \leq \frac{5\epsilon'^2}{2(\lambda - \rho)}.
        \]
        Since the accuracy $\epsilon$ of the subproblem in Algorithm~\ref{alg-subgdmax}
        is also defined as $\mathcal{H}(\phi, \Q^{*})-\mathcal{H}(\phi, \Q_{I})
        \leq \epsilon$, to ensure that
        $4\rho\epsilon \leq \frac{\varepsilon^{2}}{2}$, it suffices to require
        \[
            4\rho\cdot \frac{5\epsilon'^2}{2(\lambda - \rho)}\leq\frac{\varepsilon^{2}}{2}.
        \]
        This leads us to choose $\epsilon' = \mathcal{O}(\varepsilon)$.
        By combining these two components, we obtain the desired oracle complexity
        results. Specifically, the number of subgradient evaluations is equal to
        the number of outer iterations, and the inexact JKO steps
        require a product of the number of outer iterations $K$ and the number
        of inner iterations $I$.

\section{Proofs for Section \ref{sec:smooth}}

    \subsection*{Proof of Proposition \ref{prop:danskin}}
        Since $\mathcal{V}(\phi)=\max_{\Q \in \mathcal{P}_2}\mathcal{H}(\phi,\Q)=\max_{\Q \in \mathcal{P}_2(\mathcal{X})}\mathcal{H}(\phi,\Q)$ for any $\phi\in\Phi$ from Assumption  \ref{f_assump_smooth}, we focus on the function $\mathcal{H}(\phi,\cdot)$ with the space $\mathcal{P}_2(\mathcal{X})$.
        By Assumption~\ref{f_assump_sc} and Lemma~\ref{lem:Haggcvx} we know that
        $\mathcal{H}(\phi, \cdot)$ is strongly concave a.g.g. centered at
        $\nu \in \mathcal{P}_{2}^{r}$ for each $\phi \in \Phi$. Consequently, by
        definition, the set $\argmax_{\Q \in \mathcal{P}_2}\mathcal{H}(\phi, \Q)$
        is a singleton. On the other hand, from Assumption \ref{f_assump_smooth}
        we know that
        $\mathcal{H}(\cdot,\Q)=\mathbb{E}_{\xi \sim \Q}[\ell(f_{(\cdot)}, \xi)]-\lambda
        \mathcal{D}(\Q,\bP)$
        is differentiable for every $\Q\in\mathcal{P}_{2}$, and $\nabla_{\phi}\mathcal{H}
        (\phi,\Q)=\mathbb{E}_{\xi \sim \Q}[\nabla_{\phi}\ell(f_{\phi}, \xi)]$ is
        Lipschitz continuous with respect to $\phi\times \Q\in\Phi\times \mathcal{P}
        _{2}$. We then focus on applying Danskin's theorem \cite[Theorem~9.26]{shapiro2021lectures},
        which extends Danskin's theorem to general compact topological spaces. In order
        to derive the desired results, it suffices to verify that the space $\mathcal{P}
        _{2}(\mathcal{X})$ is a compact topological space.

        To proceed, we prove that the space of probability measures $\mathcal{P}_{2}
        (\mathcal{X})$ equipped with the Wasserstein-2 metric $\mathcal{W}_{2}$ is
        compact. Let $\{\mu_{i}\}_{i \geq 1}$ be an arbitrary sequence in
        $\mathcal{P}_{2}(\mathcal{X})$. Since $\mathcal{X}$ is compact, every
        sequence of probability measures on $\mathcal{X}$ is tight. By Prokhorov's
        theorem \cite[Theorem 2.4]{van2000asymptotic}, the sequence $\{\mu_{i}\}$ has a convergent subsequence that
        converges weakly to some probability measure
        $\mu \in \mathcal{P}_{2}(\mathcal{X})$. Then there exists a probability space
        $(\Omega, \mathcal{F}, \nu)$ and random variables $X, X_{1}, X_{2}, \ldots$
        with distribution $\mu,\mu_{1},\mu_{2},\ldots$ such that $\lim_{i
        \rightarrow \infty}\|X_{i}- X\| = 0$ almost surely under $\nu$ by Skorokhod's
        representation theorem. Since $\mathcal{X}$ is compact, there exists a constant
        $c > 0$ such that $\|X_{i}- X\| \leq c$ for all $i \geq 1$. Hence, we know
        from the dominated convergence theorem that $\lim_{i\rightarrow\infty}\mathbb{E}
        [\|X_{i}-X\|^{2}]= 0.$ Consequently, for the couplings
        $\pi_{i}\in \Pi(\mu_{i}, \mu)$ one has that
        \[
            \lim_{i\rightarrow\infty}\int_{\mathcal{X} \times \mathcal{X}}\|x_i-x\|
            ^{2}\mathrm{~d}\pi_i(x_i, x) = 0.
        \]
        Taking the infimum of the left-hand side over $\Pi(\mu_{i}, \mu)$ and
        recalling the definition of the Wasserstein distance, we deduce that
        $\lim_{i\rightarrow\infty}\mathcal{W}_{2}(\mu_{i}, \mu) = 0$, which
        means that $\mu$ is also a limit point of $\{\mu_{i}\}_{i \geq 1}$ for
        the Wasserstein-$2$ metric. Then $(\mathcal{P}_{2}(\mathcal{X}), \mathcal{W}
        _{2})$ is compact, which completes the proof.

    \subsection*{Proof of Lemma \ref{lem:q-lipschitz}}
        From Lemma \ref{lem:Haggcvx} we know that the function $\mathcal{H}(\phi,\cdot
        )$ is $(\lambda-L)$-strongly concave a.g.g. centered at
        $\nu \in \mathcal{P}_{2}^{r}$ in $\mathcal{P}_{2}$. Then
        \begin{align}
             & -\mathcal{H}(\phi_{1},\Q^{*}(\phi_{2}))+\mathcal{H}(\phi_{1},\Q^{*}(\phi_{1}))\geq \frac{\lambda-L}{2}\cdot\mathcal{W}_{\nu}^{2}(\Q^{*}(\phi_{1}),\Q^{*}(\phi_{2})),\label{q-strongcvxineq1} \\
             & -\mathcal{H}(\phi_{2},\Q^{*}(\phi_{1}))+\mathcal{H}(\phi_{2},\Q^{*}(\phi_{2}))\geq \frac{\lambda-L}{2}\cdot\mathcal{W}_{\nu}^{2}(\Q^{*}(\phi_{1}),\Q^{*}(\phi_{2})).\label{q-strongcvxineq2} 
        \end{align}
        Moreover, by the $L$-Lipschitz continuity of
        $\nabla_{\phi}\mathcal{H}(\cdot, \Q)$ for any $\Q \in \mathcal{P}_{2}$
        we have
        \begin{equation}
            \begin{aligned}
                \label{q-yconcave}-\mathcal{H}(\phi_{2},\Q^{*}(\phi_{1}))+\mathcal{H}(\phi_{1},\Q^{*}(\phi_{1}))\leq \langle \nabla_{\phi}\mathcal{H}(\phi_{1},\Q^{*}(\phi_{1})), \phi_{1}-\phi_{2}\rangle+\frac{L}{2}\|\phi_{2}-\phi_{1}\|^{2} 
            \end{aligned}
        \end{equation}
        and
        \begin{equation}
            \begin{aligned}
                \label{q-ylip}-\mathcal{H}(\phi_{1},\Q^{*}(\phi_{2}))+\mathcal{H}(\phi_{2},\Q^{*}(\phi_{2}))\leq \langle \nabla_{\phi}\mathcal{H}(\phi_{2},\Q^{*}(\phi_{2})), \phi_{2}-\phi_{1}\rangle+\frac{L}{2}\|\phi_{2}-\phi_{1}\|^{2}. 
            \end{aligned}
        \end{equation}
        Incorporating \eqref{q-strongcvxineq1}--\eqref{q-ylip}, we obtain
        \begin{equation}
            \label{q-perturbe-key}
            \begin{aligned}
                 & (\lambda-L)\cdot\mathcal{W}_{\nu}^{2}(\Q^{*}(\phi_{1}),\Q^{*}(\phi_{2})) \leq \langle\nabla_{\phi}\mathcal{H}(\phi_{1},\Q^{*}(\phi_{1}))-\nabla_{\phi}\mathcal{H}(\phi_{2},\Q^{*}(\phi_{2})), \phi_{1}-\phi_{2}\rangle +L\|\phi_{2}-\phi_{1}\|^{2}.
            \end{aligned}
        \end{equation}
        Recall that
        $\mathcal{H}(\phi,\Q)=\mathbb{E}_{\xi \sim \mathbb{Q}}[\ell(f_{\phi}, \xi
        )]-\lambda\cdot \mathcal{D}(\Q,\bP)$,
        it follows for any transport map $T$ satisfying $T_{\#}\nu=\Q$ that
        \[
            \nabla_{\phi}\mathcal{H}(\phi,\Q)=\mathbb{E}_{\xi \sim \mathbb{Q}}[\nabla
            _{\phi}\ell(f_{\phi}, \xi)]=\mathbb{E}_{\xi \sim \nu}[\nabla_{\phi}
            \ell(f_{\phi}, T(\xi))].
        \]
        Since $\nabla_{\phi}\ell(f_{\phi}, \cdot)$ is $L$-Lipschitz continuous,
        we have for the OT maps $T_{1}$, $T_{2}$ satisfy $(T_{1})_{\#}\nu=\Q^{*}(
        \phi_{1})$ and $(T_{2})_{\#}\nu=\Q^{*}(\phi_{2})$ that
        \[
            \|\nabla_{\phi}\ell(f_{\phi}, T_{1}(\xi))-\nabla_{\phi}\ell(f_{\phi},
            T_{2}(\xi))\|\leq L\|T_{1}(\xi)-T_{2}(\xi)\|\quad \text{for } \nu\text{-a.e.}\ \xi.
        \]
        Integrating with respect to $\nu$, we derive
        \begin{equation}
            \label{eq:key_distri_perturb}
            \begin{aligned}
                       & \|\nabla_{\phi}\mathcal{H}(\phi,\Q^{*}(\phi_{1}))-\nabla_{\phi}\mathcal{H}(\phi,\Q^{*}(\phi_{2}))\|                    \\
                =\     & \|\mathbb{E}_{\xi\sim\nu}[\nabla_{\phi}\ell(f_{\phi}, T_{1}(\xi))-\nabla_{\phi}\ell(f_{\phi}, T_{2}(\xi))]\|           \\
                \leq\  & \mathbb{E}_{\xi\sim\nu}[\|\nabla_{\phi}\ell(f_{\phi}, T_{1}(\xi))-\nabla_{\phi}\ell(f_{\phi}, T_{2}(\xi))\|]           \\
                \leq\  & L\cdot\mathbb{E}_{\xi\sim\nu}[\|T_{1}(\xi)-T_{2}(\xi)\|]\leq L\cdot\mathcal{W}_{\nu}(\Q^{*}(\phi_{1}),\Q^{*}(\phi_{2})),
            \end{aligned}
        \end{equation}
        where the last inequality is from Jensen's inequality. Then it
        follows that
        \begin{align*}
            \|\nabla_{\phi}\mathcal{H}(\phi_{1},\Q^{*}(\phi_{1}))-\nabla_{\phi}\mathcal{H}(\phi_{2},\Q^{*}(\phi_{2}))\|\leq L(\|\phi_{1}-\phi_{2}\|+\mathcal{W}_{\nu}(\Q^{*}(\phi_{1}),\Q^{*}(\phi_{2}))). 
        \end{align*}
        This together with \eqref{q-perturbe-key} implies that
        \begin{align*}
            (\lambda-L)\cdot\mathcal{W}_{\nu}^{2}(\Q^{*}(\phi_{1}),\Q^{*}(\phi_{2})) \leq L\cdot\mathcal{W}_{\nu}(\Q^{*}(\phi_{1}),\Q^{*}(\phi_{2}))\cdot\|\phi_{1}-\phi_{2}\|+2L\|\phi_{1}-\phi_{2}\|^{2}.
        \end{align*}
        Let $\beta:=\frac{\mathcal{W}_{\nu}(\Q^{*}(\phi_{1}),\Q^{*}(\phi_{2}))}{\|\phi_{1}-\phi_{2}\|}$.
        Then, the above inequality gives
        \begin{align*}
            \beta^{2}\leq \frac{L}{\lambda-L}\beta+\frac{2L}{\lambda-L}\leq \frac{1}{2}\beta^{2}+\frac{L^{2}}{2(\lambda-L)^{2}}+\frac{2L}{\lambda-L} & \leq \frac{1}{2}\beta^{2}+\frac{L^{2}+4L(\lambda-L)}{2(\lambda-L)^{2}}   \\
            & \leq \frac{1}{2}\beta^{2}+\frac{(L+2(\lambda-L))^{2}}{2(\lambda-L)^{2}},
        \end{align*}
        where the second inequality holds because $ab\leq \frac{1}{2}(a^{2}+b^{2}
        )$ for any $a,b\in\R$. Thus, we get
        \[
            \mathcal{W}_{2}(\Q^{*}(\phi_{1}),\Q^{*}(\phi_{2}))\leq \mathcal{W}_{\nu}
            (\Q^{*}(\phi_{1}),\Q^{*}(\phi_{2})) \leq \frac{2\lambda-L}{\lambda-L}
            \cdot \|\phi_{1}-\phi_{2}\|.
        \]
        Furthermore, with the help of the Danskin's theorem, i.e., Proposition \ref{prop:danskin},
        we know that
        \[
            \|\nabla \mathcal{V}(\phi_{1})-\nabla \mathcal{V}(\phi_{2})\|=\|\nabla
            _{\phi}\mathcal{H}(\phi_{1}, \Q^{*}(\phi_{1}))-\nabla_{\phi}\mathcal{H}
            (\phi_{2}, \Q^{*}(\phi_{2}))\| \leq L\left(1+\frac{2\lambda-L}{\lambda-L}
            \right)\|\phi_{1}-\phi_{2}\|.
        \]
        The proof is complete.

    \subsection*{Proof of Theorem \ref{thm:smooth_oracle_complexity}}
        First we observe from Lemma \ref{lem:q-lipschitz} that
        \begin{equation}
            \label{eq:smooth_des}
            \begin{aligned}
                \mathcal{V}(\phi_{k+1}) 
                 & \leq \mathcal{V}(\phi_{k})+\langle\nabla_{\phi}\mathcal{H}(\phi_{k},\Q^{*}_{k}), \phi_{k+1}-\phi_{k}\rangle+\kappa_2L\|\phi_{k+1}-\phi_{k}\|^{2},
            \end{aligned}
        \end{equation}
        where $\Q_{k}^{*}:=\argmax_{\Q\in\mathcal{P}_2}\mathcal{H}(\phi_{k},\Q)$.
        The update
        $\phi_{k+1}=\proj_{\Phi}(\phi_{k}-\eta_k\nabla_{\phi}\mathcal{H}(\phi_{k},\Q
        _{k}))$
        together with its optimality by projection implies that
        \[
            \langle\nabla_{\phi}\mathcal{H}(\phi_{k},\Q_{k}), \phi_{k+1}-\phi_{k}
            \rangle+\frac{1}{2\eta_k}\|\phi_{k+1}-\phi_{k}\|^{2}\leq0.
        \]
        Then combining this with \eqref{eq:smooth_des} we have that
        \begin{align}
            \label{eq:smooth_key} & \mathcal{V}(\phi_{k+1}) \notag     \\
            \leq\                 & \mathcal{V}(\phi_{k})+\langle\nabla_{\phi}\mathcal{H}(\phi_{k},\Q^{*}_{k})-\nabla_{\phi}\mathcal{H}(\phi_{k},\Q_{k})+\nabla_{\phi}\mathcal{H}(\phi_{k},\Q_{k}), \phi_{k+1}-\phi_{k}\rangle +\kappa_2L\|\phi_{k+1}-\phi_{k}\|^{2}\notag \\
            \leq\                 & \mathcal{V}(\phi_{k})+\langle\nabla_{\phi}\mathcal{H}(\phi_{k},\Q^{*}_{k})-\nabla_{\phi}\mathcal{H}(\phi_{k},\Q_{k}), \phi_{k+1}-\phi_{k}\rangle+\left(\frac{(3\lambda-2L)L}{2(\lambda-L)}-\frac{1}{2\eta_k}\right)\|\phi_{k+1}-\phi_{k}\|^{2}\notag               \\
            \leq\                 & \mathcal{V}(\phi_{k})-\left(\frac{1}{2\eta_k}-\frac{(3\lambda-2L)L}{2(\lambda-L)}-\frac{L}{2}\right)\|\phi_{k+1}-\phi_{k}\|^{2}+\frac{L}{2}\mathcal{W}_{\nu}^{2}(\Q_{k}^{*},\Q_{k}),
        \end{align}
        where the last inequality is from similar argument for deriving gradient
        Lipschitz property \eqref{eq:key_distri_perturb}.

        On the other hand, by the definition of $\Q^{*}_{k}$ and the subproblem of
        Algorithm \ref{alg-subgdmax}, we know that
        \[
            \mathcal{H}(\phi_{k},\Q_{k})\geq \mathcal{H}(\phi_{k},\Q_{k}^{*}) -\epsilon
            .
        \]
        This together with the a.g.g. strong concavity with respect to $\Q$ centered
        at $\nu \in \mathcal{P}_{2}^{r}$ implies that
        \[
            \epsilon\geq \mathcal{H}(\phi_{k},\Q^{*}_{k}) -\mathcal{H}(\phi_{k},\Q
            _{k})\ge\frac{\lambda-L}{2}\cdot\mathcal{W}_{\nu}^{2}(\Q_{k},\Q^{*}_{k}
            ),
        \]
        and consequently we know from \eqref{eq:smooth_key} that
        \begin{equation*}
            \mathcal{V}(\phi_{k+1}) \leq \mathcal{V}(\phi_{k})-\left(\frac{1}{2\eta_k}
            -\frac{(3\lambda-2L)L}{2(\lambda-L)}-\frac{L}{2}\right)\|\phi_{k+1}-\phi
            _{k}\|^{2}+\frac{L\epsilon}{\lambda-L}.
        \end{equation*}
        Taking the sum over $k$, we obtain
        \[
            \mathcal{V}(\phi_{K}) \leq \mathcal{V}(\phi_{0})-\left(\frac{1}{2\eta_k}
            -\frac{(4\lambda-3L)L}{2(\lambda-L)}\right) \sum_{k=0}^{K-1}\|\phi_{k+1}
            -\phi_{k}\|^{2}+\frac{L\epsilon}{\lambda-L}\cdot K.
        \]
        Rearranging this, then we have
        \begin{equation*}
            \frac{1}{K}\sum_{k=0}^{K-1}\|\phi_{k+1}-\phi_{k}\|^{2}\leq \frac{2\eta_k(\lambda-L)\cdot(\mathcal{V}(\phi_{0})-\min_{\phi}\mathcal{V}(\phi))}{(\lambda-L-\eta_k
            L(4\lambda-3L))K}
            +\frac{2\eta_k L\epsilon}{\lambda-L-\eta_k L(4\lambda-3L)}.
        \end{equation*}
        In the meantime, by the optimality condition of the projection step
$\phi_{k+1}=\proj_\Phi(\phi_k-\eta_k \nabla_\phi\mathcal H(\phi_k,\Q_k))$, we have
\[
\frac{\phi_k-\phi_{k+1}}{\eta_k}-\nabla_\phi\mathcal H(\phi_k,\Q_k)
\in
\partial\iota_\Phi(\phi_{k+1}).
\]
Therefore,
\[
\begin{aligned}
& \dist^{2}\left(0,\nabla \mathcal{V}(\phi_{k+1})
+\partial\iota_{\Phi}(\phi_{k+1})\right) \\
\leq\;&
\left\|
\nabla\mathcal V(\phi_{k+1})
+\frac{\phi_k-\phi_{k+1}}{\eta_k}
-\nabla_\phi\mathcal H(\phi_k,\Q_k)
\right\|^2 \\
\leq\;&
2\left(\frac{1}{\eta_k}+2\kappa_2L\right)^2
\|\phi_{k+1}-\phi_k\|^2
+
2\|\nabla_\phi\mathcal H(\phi_k,\Q_k)-\nabla_\phi\mathcal H(\phi_k,\Q_k^*)\|^2 \\
\leq\;&
2\left(\frac{1}{\eta_k}+2\kappa_2L\right)^2
\|\phi_{k+1}-\phi_k\|^2
+
\frac{4L^2\epsilon}{\lambda-L}.
\end{aligned}
\]
Then letting $\eta_k=\frac{\lambda-L}{2L(4\lambda-3L)}$ we have
\[
\begin{aligned}
& \min_{k\in[K-1]}
\left\{
\dist^{2}\left(
0,\nabla \mathcal{V}(\phi_{k+1})
+\partial\iota_{\Phi}(\phi_{k+1})
\right)
\right\} \\
\leq\;&
4\left(\frac{1}{\eta_k}+2\kappa_2 L\right)^{2}
\frac{
\eta_k(\lambda-L)
\left(\mathcal{V}(\phi_0)-\min_{\phi}\mathcal{V}(\phi)\right)
}{
\left(\lambda-L-\eta_k L(4\lambda-3L)\right)K
} +
\frac{
4\left(\frac{1}{\eta_k}+2\kappa_2 L\right)^2\eta_k L\epsilon
}{
\lambda-L-\eta_k L(4\lambda-3L)
}
+
\frac{4L^2\epsilon}{\lambda-L}\\
\leq\;&
\frac{
4L(11\lambda-8L)^2
\left(\mathcal{V}(\phi_0)-\min_{\phi}\mathcal{V}(\phi)\right)
}{
(\lambda-L)(4\lambda-3L)K
} 
+
\left(
\frac{(11\lambda-8L)L}{\lambda-L}
\right)^2
\frac{4\epsilon}{4\lambda-3L}
+
\frac{4L^2\epsilon}{\lambda-L}.
\end{aligned}
\]
The oracle complexity results follow from a similar analysis as in Theorem
        \ref{thm:oraclecomplexity}. The proof is complete.

\subsection*{Auxiliary Lemmas for Convergence Analysis of Algorithm~\ref{alg:single}}

The convergence proof of Algorithm \ref{alg:single} relies primarily on a basic descent property, along with an exponential bound on the growth of the associated increasing term.

\begin{lemma}[Basic descent property]\label{lem:singledescent_key1}
Suppose that Assumptions \ref{f_assump_sc} and \ref{f_assump_smooth} hold. For Algorithm \ref{alg:single} with $\eta\leq 1 / 32\kappa_2 L$, 
the iterates $\{\phi_k\}_{k \geq 0}$ satisfy
$$
\mathcal{V}(\phi_{k+1}) \leq \mathcal{V}(\phi_{k})-\frac{7 \eta}{16}\left\|\nabla \mathcal{V}(\phi_{k})\right\|^2+\frac{9 \eta L^2}{16} \|T_{\nu}^{k+1}-T_{\nu}^*(\phi_{k})\|_{\nu}^2,
$$
where $T_{\nu}^*(\phi_{k})$ is the OT map from $\nu$ to $\Q^*(\phi_k)$.
\end{lemma} 
    
\begin{proof} 
Since $\mathcal{V}$ is $2\kappa_2 L$-gradient Lipschitz from Lemma \ref{lem:q-lipschitz}, we have
\begin{equation}
\label{eq:vlsmooth}
\mathcal{V}(\phi_{k+1}) - \mathcal{V}(\phi_{k})-\left(\phi_{k+1}- \phi_{k}\right)^{\top} \nabla \mathcal{V}(\phi_{k}) \leq \kappa_2 L\|\phi_{k+1}- \phi_{k}\|^2 .
\end{equation}
Plugging $\phi_{k+1}- \phi_{k}= - \eta\cdot\nabla_{\phi} \mathcal{H}(\phi_{k},\Q_{k+1})$ into \eqref{eq:vlsmooth} yields that
$$
\begin{aligned}
\mathcal{V}(\phi_{k+1})  
\leq\ &  \mathcal{V}(\phi_{k})-\eta\|\nabla \mathcal{V}(\phi_{k})\|^2+\eta^2 \kappa_2 L\left\|\nabla_{\phi} \mathcal{H}(\phi_{k},\Q_{k+1})\right\|^2 \\
& +\eta\left(\nabla \mathcal{V}(\phi_{k})-\nabla_{\phi} \mathcal{H}(\phi_{k},\Q_{k+1})\right)^{\top} \nabla \mathcal{V}(\phi_{k}).
\end{aligned}
$$
Using the Young's inequality, we have
$$
\left(\nabla \mathcal{V}(\phi_{k})-\nabla_{\phi} \mathcal{H}(\phi_{k},\Q_{k+1})\right)^{\top} \nabla \mathcal{V}(\phi_{k}) \leq \frac{\left\|\nabla \mathcal{V}(\phi_{k})-\nabla_{\phi} \mathcal{H}(\phi_{k},\Q_{k+1})\right\|^2+\left\|\nabla \mathcal{V}(\phi_{k})\right\|^2}{2} .
$$
By the Cauchy--Schwarz inequality, we have
$$
\left\|\nabla_{\phi} \mathcal{H}(\phi_{k},\Q_{k+1})\right\|^2 \leq 2\left(\left\|\nabla \mathcal{V}(\phi_{k})-\nabla_{\phi} \mathcal{H}(\phi_{k},\Q_{k+1})\right\|^2+\left\|\nabla \mathcal{V}(\phi_{k})\right\|^2\right).
$$
Combining these inequalities yields
\begin{align}
\label{eq:des-key12}
\mathcal{V}(\phi_{k+1}) 
\leq\ & \mathcal{V}(\phi_{k})-\left(\frac{\eta}{2}- 2\eta^2 \kappa_2 L\right)\left\|\nabla \mathcal{V}(\phi_{k})\right\|^2+\left(\frac{\eta}{2}+2\eta^2 \kappa_2 L\right)\left\|\nabla \mathcal{V}(\phi_{k})-\nabla_{\phi} \mathcal{H}(\phi_{k},\Q_{k+1})\right\|^2.
\end{align}
Since  $\eta\leq 1 / 32\kappa_2 L$, we know that
\begin{equation}\label{eq:stepcalcu}
\frac{7 \eta}{16} \leq \frac{\eta}{2}-2 \eta^2 \kappa_2 L  \leq \frac{\eta}{2}+2 \eta^2 \kappa_2  L\leq \frac{9 \eta}{16}.
\end{equation}
 Combining \eqref{eq:des-key12} and \eqref{eq:stepcalcu}, it follows that
\begin{align*}
\mathcal{V}(\phi_{k+1}) 
\leq\ & \mathcal{V}(\phi_{k})-\frac{7 \eta}{16}\left\|\nabla \mathcal{V}(\phi_{k})\right\|^2+\frac{9 \eta}{16}\left\|\nabla \mathcal{V}(\phi_{k})-\nabla_{\phi} \mathcal{H}(\phi_{k},\Q_{k+1})\right\|^2.
\end{align*}
Since $\nabla \mathcal{V}(\phi_{k})=\nabla_{\phi} \mathcal{H}(\phi_{k},\Q^*(\phi_{k}))$, we have
$$
\left\|\nabla \mathcal{V}(\phi_{k})-\nabla_{\phi} \mathcal{H}(\phi_{k},\Q_{k+1})\right\|^2\leq L^2\cdot\mathcal{W}_{2}^2(\Q_{k+1},\Q^*(\phi_{k})) \leq L^2 \|T_{\nu}^{k+1}-T_{\nu}^*(\phi_{k})\|_{\nu}^2.
$$
Putting these pieces together yields the desired inequality.
\end{proof}

\begin{lemma}[Exponential error decay]\label{lem:singledescent_key2}
Suppose that Assumptions \ref{f_assump_sc}, \ref{ass:learningerror}, and \ref{f_assump_smooth} hold. For Algorithm \ref{alg:single} with $\eta\leq\frac{1-\tau}{16\kappa_1 L}$ and $\gamma>\frac{36}{\lambda-\rho}$, 
it follows that
$$
\|T_{\nu}^{k+1}-T_{\nu}^*(\phi_{k+1})\|_{\nu}^2 \leq\left(\frac{1+\tau}{2}+\frac{4\kappa_1^2 L^2 \eta^2}{1-\tau} \right) \|T_{\nu}^{k}-T_{\nu}^*(\phi_{k})\|_{\nu}^2+\frac{4\kappa_1^2 \eta^2}{1-\tau} \left\|\nabla \mathcal{V}(\phi_{k})\right\|^2+\frac{4\gamma\epsilon'^2}{\lambda-\rho},
$$
where $\tau\in(0,1)$ is defined by Proposition \ref{prop:dualppm}.
\end{lemma}
    
\begin{proof}
From the JKO update in Proposition \ref{prop:dualppm} we know that
$$
\begin{aligned}
\|T_{\nu}^{k+1}-T_{\nu}^*(\phi_{k})\|_{\nu}^2\leq\tau\cdot\left(\|T_{\nu}^{k}-T_{\nu}^*(\phi_{k})\|_{\nu}^2+\frac{
4 \gamma\epsilon'^2}{\lambda-\rho}\right).
\end{aligned}
$$
Using the Young's inequality, we have
$$
\begin{aligned}
\|T_{\nu}^{k+1}-T_{\nu}^*(\phi_{k+1})\|_{\nu}^2 & \leq\left(1+\frac{1-\tau}{2\tau}\right)\|T_{\nu}^{k+1}-T_{\nu}^*(\phi_{k})\|_\nu^2+\left(1+\frac{2\tau}{1-\tau}\right)\left\|T_{\nu}^*(\phi_{k+1})-T_{\nu}^*(\phi_{k})\right\|_{\nu}^2 \\
& \leq\frac{1+\tau}{2 \tau}\cdot\|T_{\nu}^{k+1}-T_{\nu}^*(\phi_{k})\|_{\nu}^2+\frac{2}{1-\tau}\cdot\left\|T_{\nu}^*(\phi_{k+1})-T_{\nu}^*(\phi_{k})\right\|_{\nu}^2 \\
& \leq \frac{1+\tau}{2} \cdot \|T_{\nu}^{k}-T_{\nu}^*(\phi_{k})\|_{\nu}^2+\frac{2}{1-\tau}\cdot\left\|T_{\nu}^*(\phi_{k+1})-T_{\nu}^*(\phi_{k})\right\|_{\nu}^2+\frac{2(1+\tau) \gamma\epsilon'^2}{\lambda-\rho}.
\end{aligned}
$$
Since $\Q^*(\phi)$ is $\kappa_1$-Lipschitz from Lemma \ref{lem:q-lipschitz}, we have
\[
\mathcal{W}_{2}(\Q^*(\phi_{k+1}),\Q^*(\phi_{k}))\leq \left\|T_{\nu}^*(\phi_{k+1})-T_{\nu}^*(\phi_{k})\right\|_{\nu} \leq  \kappa_1\cdot \|\phi_{k+1}-\phi_{k}\|.
\]
Furthermore, since $\phi_{k+1}- \phi_{k}= -\eta\cdot\nabla_{\phi} \mathcal{H}(\phi_{k},\Q_{k+1})$, we know from \eqref{eq:key_distri_perturb} that
$$
\|\phi_{k+1}- \phi_{k}\|^2=\eta^2\left\|\nabla_{\phi} \mathcal{H}(\phi_{k},\Q_{k+1})\right\|^2 \leq 2 \eta^2 L^2 \|T_{\nu}^{k+1}-T_{\nu}^*(\phi_{k})\|_{\nu}^2+2 \eta^2\left\|\nabla \mathcal{V}(\phi_{k})\right\|^2.
$$
Combining the above estimates, we obtain
\begin{align*}
\|T_{\nu}^{k+1}-T_{\nu}^*(\phi_{k+1})\|_{\nu}^2 
& \leq \frac{1+\tau}{2} \cdot \|T_{\nu}^{k}-T_{\nu}^*(\phi_{k})\|_{\nu}^2+\frac{4\kappa_1^2\eta^2}{1-\tau}\left( L^2 \|T_{\nu}^{k+1}-T_{\nu}^*(\phi_{k})\|_{\nu}^2+\left\|\nabla \mathcal{V}(\phi_{k})\right\|^2\right)\\
&\quad +\frac{2(1+\tau) \gamma\epsilon'^2}{\lambda-\rho}.
\end{align*}
Substituting the JKO contraction bound $\|T_{\nu}^{k+1}-T_{\nu}^*(\phi_{k})\|_{\nu}^2\leq\tau(\|T_{\nu}^{k}-T_{\nu}^*(\phi_{k})\|_{\nu}^2+\frac{4 \gamma\epsilon'^2}{\lambda-\rho})$ into the above inequality, we obtain
\begin{align*}
\|T_{\nu}^{k+1}-T_{\nu}^*(\phi_{k+1})\|_{\nu}^2 
& \leq \left(\frac{1+\tau}{2}+\frac{4\kappa_1^2 L^2 \eta^2\tau}{1-\tau}\right) \|T_{\nu}^{k}-T_{\nu}^*(\phi_{k})\|_{\nu}^2+\frac{4\kappa_1^2 \eta^2}{1-\tau} \left\|\nabla \mathcal{V}(\phi_{k})\right\|^2\\
&\quad +\left(2(1+\tau)+\frac{16\kappa_1^2\eta^2 L^2\tau}{1-\tau}\right)\frac{\gamma\epsilon'^2}{\lambda-\rho}.
\end{align*}
For the $\epsilon'$ coefficient, we verify that $2(1+\tau)+\frac{16\kappa_1^2\eta^2 L^2\tau}{1-\tau}\leq 4$, which requires $\frac{16\kappa_1^2\eta^2 L^2\tau}{1-\tau}\leq 2(1-\tau)$.
Indeed, from the stepsize condition $\eta\leq \frac{1-\tau}{16\kappa_1 L}$, we have $\kappa_1^2\eta^2 L^2\leq \frac{(1-\tau)^2}{256}$, so that
$$
\frac{16\kappa_1^2\eta^2 L^2\tau}{1-\tau}\leq \frac{16\tau(1-\tau)}{256}=\frac{\tau(1-\tau)}{16}<2(1-\tau).
$$
Therefore, we arrive at
\begin{align*}
\|T_{\nu}^{k+1}-T_{\nu}^*(\phi_{k+1})\|_{\nu}^2 
& \leq \left(\frac{1+\tau}{2}+\frac{4\kappa_1^2 L^2 \eta^2}{1-\tau}\right) \|T_{\nu}^{k}-T_{\nu}^*(\phi_{k})\|_{\nu}^2+\frac{4\kappa_1^2 \eta^2}{1-\tau} \left\|\nabla \mathcal{V}(\phi_{k})\right\|^2+\frac{4\gamma\epsilon'^2}{\lambda-\rho}.
\end{align*}
This completes the proof.
\end{proof}

\subsection*{Proof of Theorem \ref{thm:alternating}}
Define $\delta:=(1+\tau)/ 2 +4 \kappa_1^2 L^2 \eta^2/(1-\tau)$. Using the inequality in Lemma \ref{lem:singledescent_key2} recursively yields that
\begin{align*}
\|T_{\nu}^{k}-T_{\nu}^*(\phi_{k})\|_{\nu}^2
&\leq \delta^k C^2+\frac{4 \kappa_1^2 \eta^2}{1-\tau}\left(\sum_{j=0}^{k-1} \delta^{k-1-j}\left\|\nabla \mathcal{V}(\phi_{j})\right\|^2\right)+\frac{4\gamma\epsilon'^2}{\lambda-\rho}\cdot\sum_{j=0}^{k-1} \delta^{k-1-j},
\end{align*}
where $C:=\|T_{\nu}^{0}-T_{\nu}^*(\phi_{0})\|_{\nu}<\infty$.
 
Now we substitute this into the descent inequality of Lemma \ref{lem:singledescent_key1}. From the JKO contraction bound we have 
$$
\|T_{\nu}^{k+1}-T_{\nu}^*(\phi_{k})\|_{\nu}^2\leq\tau\left(\|T_{\nu}^{k}-T_{\nu}^*(\phi_{k})\|_{\nu}^2+\frac{4 \gamma\epsilon'^2}{\lambda-\rho}\right),
$$
and therefore Lemma \ref{lem:singledescent_key1} gives
\begin{align*}
\mathcal{V}(\phi_{k+1}) 
\leq\ & \mathcal{V}(\phi_{k})-\frac{7 \eta}{16}\left\|\nabla \mathcal{V}(\phi_{k})\right\|^2+\frac{9 \eta L^2 \tau}{16}\|T_{\nu}^{k}-T_{\nu}^*(\phi_{k})\|_{\nu}^2 +\frac{9 \eta L^2\tau}{4}\cdot\frac{\gamma\epsilon'^2}{\lambda-\rho}.
\end{align*}
Substituting the recursive bound for $\|T_{\nu}^{k}-T_{\nu}^*(\phi_{k})\|_{\nu}^2$ and noting $\tau\leq 1$, we obtain
\begin{align*}
\mathcal{V}(\phi_{k+1}) 
\leq\ & \mathcal{V}(\phi_{k})-\frac{7 \eta}{16}\left\|\nabla \mathcal{V}(\phi_{k})\right\|^2+\frac{9 \eta L^2 \delta^{k} C^2}{16}+\frac{9 \eta^3 L^2 \kappa_1^2}{4(1-\tau)}\left(\sum_{j=0}^{k-1} \delta^{k-1-j}\left\|\nabla \mathcal{V}(\phi_{j})\right\|^2\right)\\
&+ \frac{9 \eta L^2}{16} \cdot\frac{4\gamma\epsilon'^2}{\lambda-\rho}\cdot\left(1+\sum_{j=0}^{k-1} \delta^{k-1-j}\right).
\end{align*}
Summing up over $k=0,1, \ldots, K-1$ and rearranging the terms yields that
\begin{align*}
\mathcal{V}(\phi_{K}) 
\leq\ & \mathcal{V}(\phi_{0})-\frac{7 \eta}{16}\left(\sum_{k=0}^{K-1}\left\|\nabla \mathcal{V}(\phi_{k})\right\|^2\right)+\frac{9 \eta L^2 C^2}{16}\left(\sum_{k=0}^{K-1} \delta^k\right)+\frac{9 \eta^3 L^2 \kappa_1^2}{4(1-\tau)}\left(\sum_{k=0}^{K-1} \sum_{j=0}^{k-1} \delta^{k-1-j}\left\|\nabla \mathcal{V}(\phi_{j})\right\|^2\right)\\
&+ \frac{9 \eta L^2}{4}\cdot \frac{\gamma\epsilon'^2}{\lambda-\rho}\cdot\sum_{k=0}^{K-1}\left(1+\sum_{j=0}^{k-1} \delta^{k-1-j}\right).
\end{align*}
Since $\eta\leq (1-\tau)/ 16\kappa_1 L$, we have $\delta \leq \frac{3+\tau}{4}$ and $\frac{9 \eta^3 L^2 \kappa_1^2}{4(1-\tau)} \leq \frac{9 \eta(1-\tau)}{1024}$. This implies that $\sum_{k=0}^{K-1} \delta^k \leq \frac{4}{1-\tau} $,
$$
\sum_{k=0}^{K-1} \sum_{j=0}^{k-1} \delta^{k-1-j}\left\|\nabla \mathcal{V}(\phi_{j})\right\|^2 \leq \frac{4}{1-\tau}\left(\sum_{k=0}^{K-1}\left\|\nabla \mathcal{V}(\phi_{k})\right\|^2\right)
$$
and
$$
\sum_{k=0}^{K-1}\left(1+\sum_{j=0}^{k-1} \delta^{k-1-j}\right) \leq K\left(1+\frac{4}{1-\tau}\right)\leq \frac{5 K}{1-\tau}.
$$
Putting these pieces together yields that
$$
\mathcal{V}(\phi_{K}) \leq \mathcal{V}(\phi_{0})-\frac{103 \eta}{256}\left(\sum_{k=0}^{K-1}\left\|\nabla \mathcal{V}(\phi_{k})\right\|^2\right)+\frac{9 \eta L^2 C^2}{4(1-\tau)} + \frac{45 \eta L^2}{4(1-\tau)}\cdot \frac{\gamma\epsilon'^2}{\lambda-\rho}\cdot K.
$$
Rearranging the terms, we have
$$
\frac{1}{K}\left(\sum_{k=0}^{K-1}\left\|\nabla \mathcal{V}(\phi_{k})\right\|^2\right) \leq \frac{256\left(\mathcal{V}(\phi_{0})-\min_{\phi}\mathcal{V}(\phi)\right)}{103 \eta K}+\frac{576  L^2 C^2}{103(1-\tau)K}+ \frac{2880 L^2}{103(1-\tau)} \cdot\frac{\gamma\epsilon'^2}{\lambda-\rho}.
$$
Then the oracle complexity results can be directly derived. The proof is complete.

    \section{Additional Experimental Details and Figures}

    \subsection*{Additional Figures}

    \paragraph{Additional two-dimensional Gaussian mixture results.}
    Figure~\ref{fig:gm-objective-reference} reports the empirical classifier loss $\mathbb{E}_{\widehat{\mathbb{Q}}_k}[\ell]$ and empirical minimax objective $\widehat{\mathcal{H}}_k$. These curves complement the stationarity curves in Figure~\ref{fig:gm-main-convergence}.

    \begin{figure}[H]
        \centering
        \includegraphics[width=0.84\textwidth]{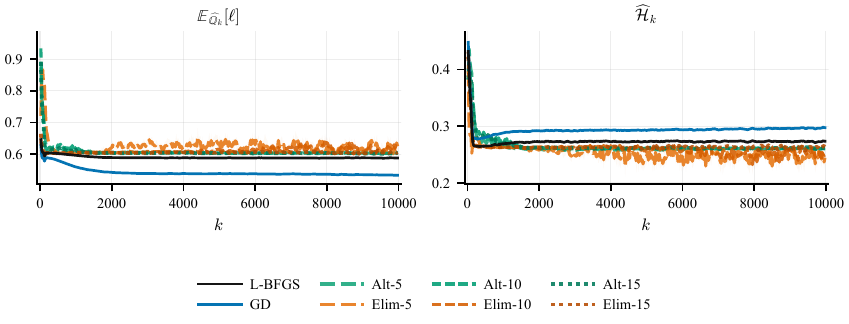}
        \caption{\textbf{Classifier loss and objective on the two-dimensional Gaussian mixture task.} The panels show the empirical classifier loss $\mathbb{E}_{\widehat{\mathbb{Q}}_k}[\ell]$ and empirical minimax objective $\widehat{\mathcal{H}}_k$. The gradient-based convergence metrics are shown separately in Figure~\ref{fig:gm-main-convergence}.}\label{fig:gm-objective-reference}
    \end{figure}

    \begin{figure}[H]
        \centering
        \includegraphics[width=0.9\textwidth]{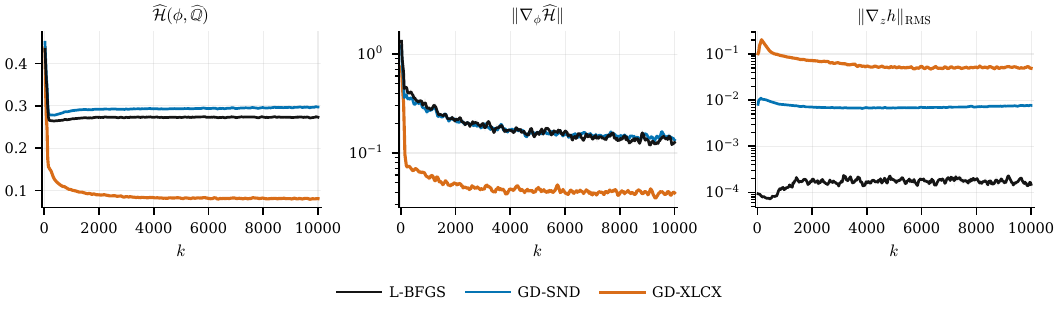}
        \caption{\textbf{Comparison of GD implementations.} GD-SND denotes the rescaled particle GD implementation following \cite{sinha2018certifying}; GD-XLCX denotes the unrescaled particle GD implementation \cite{xu2024flow}. The unrescaled implementation uses a different effective step scaling, so it is excluded from the main comparison and retained only for this implementation check.}\label{fig:gm-gd-comparison}
    \end{figure}

    \subsection*{Experimental Details}

    \paragraph{Two-dimensional Gaussian mixture.}
    The detailed hyperparameter setup can be found in Table~\ref{tab:gm-exp-details}. At each outer iteration, we draw a fresh batch from the two distributions described in Section~\ref{sec:numerical}. The inner maximization is solved only for the class-0 samples; the class-1 samples are kept fixed and are used only in the classifier update. The particle L-BFGS baseline is used for this inner finite-dimensional problem, while the classifier $f_\phi$ is trained by Adam for all methods. For the neural transport maps, we use the residual parameterization $T_\theta(x)=x+R_\theta(x)$ and initialize the map as the identity by setting the final layer of $R_\theta$ to zero.

    \begin{table}[H]
        \small
        \centering
        \caption{\textbf{Experimental setup for the two-dimensional Gaussian mixture experiment.}}\label{tab:gm-exp-details}
        \begin{tabular}{@{}llp{0.52\textwidth}@{}}
            \toprule
            Component & Hyperparameter & Value \\
            \midrule
            \multirow{2}{*}{Data} & Dimension & $2$ \\
            & Batch size & $1000$ per class \\
            \midrule
            \multirow{4}{*}{Classifier $f_\phi$} & MLP hidden layers & $3$ \\
            & MLP width & $512$ \\
            & Activation & SiLU \\
            & Optimizer & Adam, lr $10^{-4}$ \\
            \midrule
            \multirow{4}{*}{Transport map $T_\theta$} & MLP hidden layers & $3$ \\
            & MLP width & $512$ \\
            & Activation & SiLU \\
            & Optimizer & Adam, lr $10^{-3}$ \\
            \midrule
            \multirow{7}{*}{DRO} & Outer iterations & $10{,}000$ \\
            & $\lambda$ & $2\times10^{-2}$ \\
            & $\gamma$ & $5$ (Alt), $\infty$ (Elim) \\
            & Neural updates & $B\in\{5,10,15\}$ \\
            & GD steps & $15$ \\
            & L-BFGS max iter. & $100$ \\
            & L-BFGS tolerances & gtol $10^{-3}$, ftol $10^{-6}$ \\
            \bottomrule
        \end{tabular}
    \end{table}

    \paragraph{CIFAR-10 feature-space experiment.}
    The detailed hyperparameter setup can be found in Table~\ref{tab:cifar10-exp-details}. We first encode CIFAR-10 images with a fixed ViT-base encoder, and all classifiers and transport maps are trained in the feature space. A labeled sample is written as $(x,y)$, where $x$ is the feature vector and $y$ is the class label. The transport map changes only $x$ and leaves $y$ unchanged. We implement the family $\{T_{\theta,y}\}_{y=0}^{9}$ by one shared neural network: the label $y$ is mapped to a learned embedding, the embedding is concatenated with $x$, and the network outputs a residual vector denoted by $R_{\theta,y}(x)$. Thus, $T_{\theta,y}(x)=x+R_{\theta,y}(x)$. The classifier $f_\phi$ takes only the feature vector as input. We initialize $T_{\theta,y}$ as the identity map by setting the final residual layer to zero. All CIFAR-10 experiments use $\gamma=\infty$; in Neural+L-BFGS, the output of $T_{\theta,y}$ is used to initialize the L-BFGS particle solve.

    \begin{table}[H]
        \small
        \centering
        \caption{\textbf{Experimental setup for the CIFAR-10 feature-space experiment.}}\label{tab:cifar10-exp-details}
        \begin{tabular}{@{}llp{0.52\textwidth}@{}}
            \toprule
            Component & Hyperparameter & Value \\
            \midrule
            \multirow{2}{*}{Data} & Dimension & $768$ \\
            & Batch size & $500$ \\
            \midrule
            \multirow{4}{*}{Classifier $f_\phi$} & MLP hidden layers & $3$ \\
            & MLP width & $512$ \\
            & Activation & SiLU \\
            & Optimizer & Adam, lr $10^{-4}$ \\
            \midrule
            \multirow{5}{*}{Transport map $T_{\theta,y}$} & MLP hidden layers & $3$ \\
            & MLP width & $512$ \\
            & Label embedding & $512$ \\
            & Activation & SiLU \\
            & Optimizer & Adam, lr $10^{-3}$ \\
            \midrule
            \multirow{8}{*}{DRO} & Outer iterations & $2000$ \\
            & $\lambda$ & $\{10^{-3},10^{-2},10^{-1},1\}$ \\
            & $\gamma$ & $\infty$ \\
            & Neural updates & $3$ \\
            & GD steps & $15$ \\
            & L-BFGS max iter. & $100$ \\
            & L-BFGS tolerances & gtol $10^{-3}$, ftol $10^{-6}$ \\
            & PGD-50 budget & $0.2\,\overline{\|x\|}$ \\
            \bottomrule
        \end{tabular}
    \end{table}

\end{document}